\documentclass{article}

\usepackage[nonatbib, final]{neurips_2022}

\usepackage[utf8]{inputenc} 
\usepackage[T1]{fontenc}    
\usepackage{hyperref}       
\usepackage{url}            
\usepackage{booktabs}       
\usepackage{amsfonts}       
\usepackage{nicefrac}       
\usepackage{microtype}      
\usepackage{xcolor}         
\usepackage{cite}
\usepackage{subcaption} 
\usepackage{graphicx}
\usepackage{algorithmic}
\usepackage{algorithm}
\usepackage{multirow}
\usepackage{amsmath}
\usepackage{array}
\usepackage{amsthm}
\usepackage{makecell}
\usepackage{wrapfig}
\usepackage{float}
\usepackage{tabularx}
\usepackage{enumitem,kantlipsum}
\newcolumntype{P}[1]{>{\centering\arraybackslash}p{#1}}
\newcolumntype{M}[1]{>{\centering\arraybackslash}m{#1}}
\newcommand{\ie}{\textit{i.e.}\ }
\newcommand{\eg}{\textit{e.g.}\ }

\newcommand{\etc}{etc.\ }

\newcommand*\samethanks[1][\value{footnote}]{\footnotemark[#1]}
\newtheorem{theorem}{Theorem}
\newtheorem{lemma}[theorem]{Lemma}

\newtheorem{proposition}[theorem]{Proposition}

\newtheorem{definition}[theorem]{Definition}

\newcommand{\alglinelabel}{%
  \addtocounter{ALC@line}{-1}
  \refstepcounter{ALC@line}
  \label
}

\newcommand{\Comment}[2][.5\linewidth]{%
  \leavevmode\hfill\makebox[#1][l]{$\triangleright$~#2}}

\title{Learning Distributed and Fair Policies for Network Load Balancing as Markov Potential Game}

\author{%
  Zhiyuan Yao\thanks{Equal contribution.}\\
  \'Ecole Polytechnique, Cisco Systems\\
  \texttt{zhiyuan.yao@polytechnique.edu} \\
   \And
    Zihan Ding\samethanks\\
    Princeton University\\
    \texttt{zihand@princeton.edu}
}

\begin{document}

\maketitle

\begin{abstract}
This paper investigates the network load balancing problem in data centers (DCs) where multiple load balancers (LBs) are deployed, using the multi-agent reinforcement learning (MARL) framework. The challenges of this problem consist of the heterogeneous processing architecture and dynamic environments, as well as limited and partial observability of each LB agent in distributed networking systems, which can largely degrade the performance of in-production load balancing algorithms in real-world setups. Centralised-training-decentralised-execution (CTDE) RL scheme has been proposed to improve MARL performance, yet it incurs -- especially in distributed networking systems, which prefer distributed and plug-and-play design scheme -- additional communication and management overhead among agents. We formulate the multi-agent load balancing problem as a Markov potential game, with a carefully and properly designed workload distribution fairness as the potential function. A fully distributed MARL algorithm is proposed to approximate the Nash equilibrium of the game. Experimental evaluations involve both an event-driven simulator and real-world system, where the proposed MARL load balancing algorithm shows close-to-optimal performance in simulations, and superior results over in-production LBs in the real-world system.
\end{abstract}

\section{Introduction}
\label{sec:intro}

In cloud data centers (DCs) and distributed networking systems, servers are deployed on infrastructures with multiple processors to provide scalable services~\cite{dragoni2017microservices}.
To optimise workload distribution and reduce additional queuing delay, load balancers (LBs) play a significant role in such systems.
State-of-the-art network LBs rely on heuristic mechanisms~\cite{lvs, maglev, 6lb, incab2018} under the low-latency and high-throughput constraints of the data plane.
However, these heuristics are not adaptive to dynamic environments and require human interventions, which can lead to most painful mistakes in the cloud -- mis-configurations.
RL approaches have shown performance gains in distributed system and networking problems~\cite{auto2018sigcomm, decima2018, drl-udn-2019, sivakumar2019mvfst}, yet applying RL on the network load balancing problem is challenging.

First, unlike traditional workload distribution or task scheduling problem~\cite{auto2018sigcomm, decima2018}, network LBs have limited observations over the system, including task sizes and actual server load states.
Being aware of only the number of tasks they have distributed, servers can be overloaded by collided elephant tasks and have degraded quality of service (QoS).

Second, to guarantee high service availability in the cloud, multiple LBs are deployed in DCs. Network traffic is split among all LBs.
This multi-agent setup makes LBs have only partial observation over the system.

Third, modern DCs are based on heterogeneous hardware and elastic infrastructures~\cite{kumar2020fast}, where server capacities vary.
It is challenging to assign correct weights to servers according to their actual processing capacities, and this process conventionally requires human intervention -- which can lead to error-prone configurations~\cite{maglev,incab2018}.

\noindent\begin{minipage}{0.55\textwidth}
\begin{algorithm}[H]
	\footnotesize
	\caption{LB System Transition Protocol}\label{alg:model-transition}
	\begin{algorithmic}[1]
	    \STATE Initialise server load, $X_j(0) \gets 0, \forall j\in[N]$
		\FOR {each time step $t$}
		    \FOR {each LB agent $i\in[M]$}
		        \STATE Choose action $\alpha_{ij}(t)$ for coming tasks $w_i(t)$
		    \ENDFOR
		    \FOR {each server $j$}
		        \STATE Update workload:\\
		        $X_j(t)= X_j(t-1)+\sum_{i=1}^M w_i(t)\alpha_{ij}(t)-v_j(t-1)$
		    \ENDFOR
		\ENDFOR
	\end{algorithmic}
\end{algorithm}
\end{minipage}
\begin{minipage}{0.4\textwidth}
    \centering
    \includegraphics[width=\columnwidth]{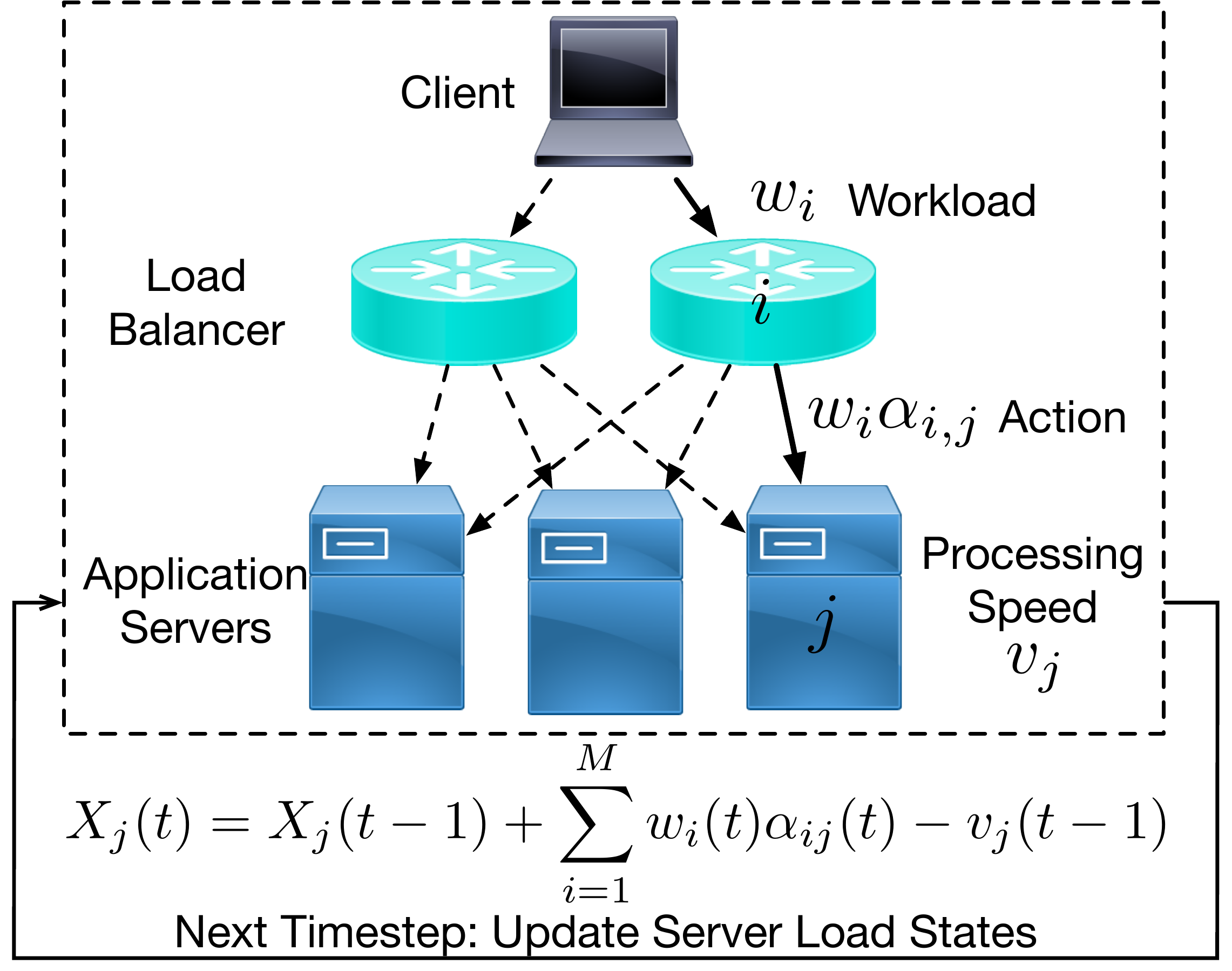}
	\captionof{figure}{Network load balancing.}
	\label{fig:intro}
\end{minipage}

Last but not least, given the low-latency and high-throughput constraints in the distributed networking setup, the interactive training procedure of RL models and the centralised-training-decentralised-execution (CTDE) scheme~\cite{foerster2018counterfactual} can incur additional communication and management overhead.

In this paper, we study the network load balancing problem in multi-agent game theoretical approach, by formulating it as a Markov potential game through specifying the proper reward function, namely variance-based fairness.
We propose a distributed Multi-Agent RL (MARL) network load balancing mechanism that is able to exploit asynchronous actions based only on local observations and inferences.
Load balancing performance gains are evaluated based on both event-based simulations and real-world experiments\footnote{Source code and data of both simulation and real-world experiment are open-sourced at \href{https://github.com/ZhiyuanYaoJ/MARLLB}{https://github.com/ZhiyuanYaoJ/MARLLB}.}.

\section{Related Work}
\label{sec:background}

\textbf{Network Load Balancing Algorithms.} The main goal of network LBs is to \textit{fairly} distribute workloads across servers.
The system transition protocol of network load balancing system is described in Alg.~\ref{alg:model-transition} and depicted in Fig.~\ref{fig:intro}.
Existing load balancing algorithms are sensitive to partial observations and inaccurate server weights.
Equal-Cost Multi-Path (ECMP) LBs randomly assign servers to new requests~\cite{glb2018, faild2018, silkroad2017}, which makes them agnostic to server load state differences.
Weighted-Cost Multi-Path (WCMP) LBs assign weights to servers proportional to their provisioned resources (\eg CPU power)~\cite{maglev, katran, concury2020, prism2020}.
However, the statically assigned weights may not correspond to the actual server processing capacity.
As depicted in Fig.~\ref{fig:motivation-multi-stage}, servers with the same IO speed yet different CPU capacities have different actual processing speed when applications have different resource requirements.
Active WCMP (AWCMP) is a variant of WCMP and it periodically probe server utilisation information (CPU/memory/IO usage)~\cite{spotlight2018, incab2018}.
However, active probing can cause delayed observations and incur additional control messages, which degrades the performance of distributed networking systems.
Local Shortest Queue (LSQ) assigns new requests to the server with the minimal number of ongoing networking connections that are \textit{locally} observed~\cite{twf2020, cheetah2020}.
It does not concern server processing capacity differences.
Shortest Expected Delay (SED) derives the ``expected delay'' as locally observed server queue length divided by statically configured server processing speed~\cite{lvs}.
However, LSQ and SED are sensitive to partial observations and misconfigurations.
As depicted in Fig.~\ref{fig:motivation-inaccurate}, the QoS performance of each load balancing algorithm degrades from the ideal setup (global observations and accurate server weight configurations) when network traffic is split across multiple LBs or server weights are mis-configured\footnote{The stochastic Markov model of the simulation is detailed in the App.~\ref{app:model-basic}}, which prevails in real-world cloud DCs.

\begin{figure}[tbp]
	\centering
	\begin{subfigure}{0.47\columnwidth}
		\centering
		\includegraphics[height=1.2in]{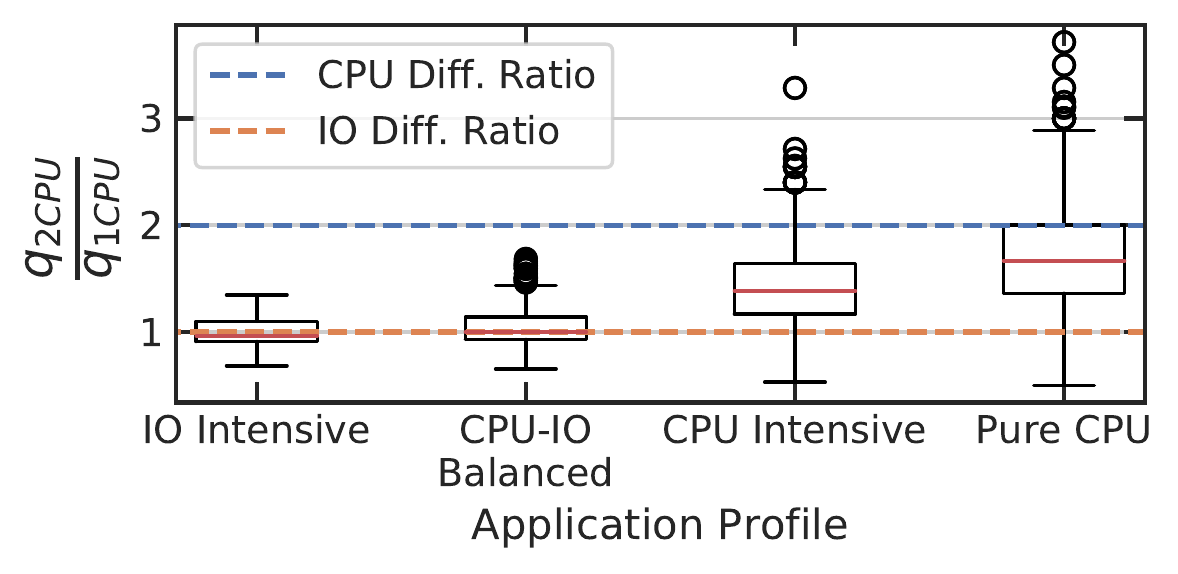}
		\caption{It is hard to accurately estimate the actual server processing speeds since it depends on both provisioned resources, and application profiles (App.~\ref{app:results-simulation-weights}).}
		\label{fig:motivation-multi-stage}
	\end{subfigure}
	\hspace{.05in}
	\begin{subfigure}{0.5\columnwidth}
		\centering
		\includegraphics[height=1.2in]{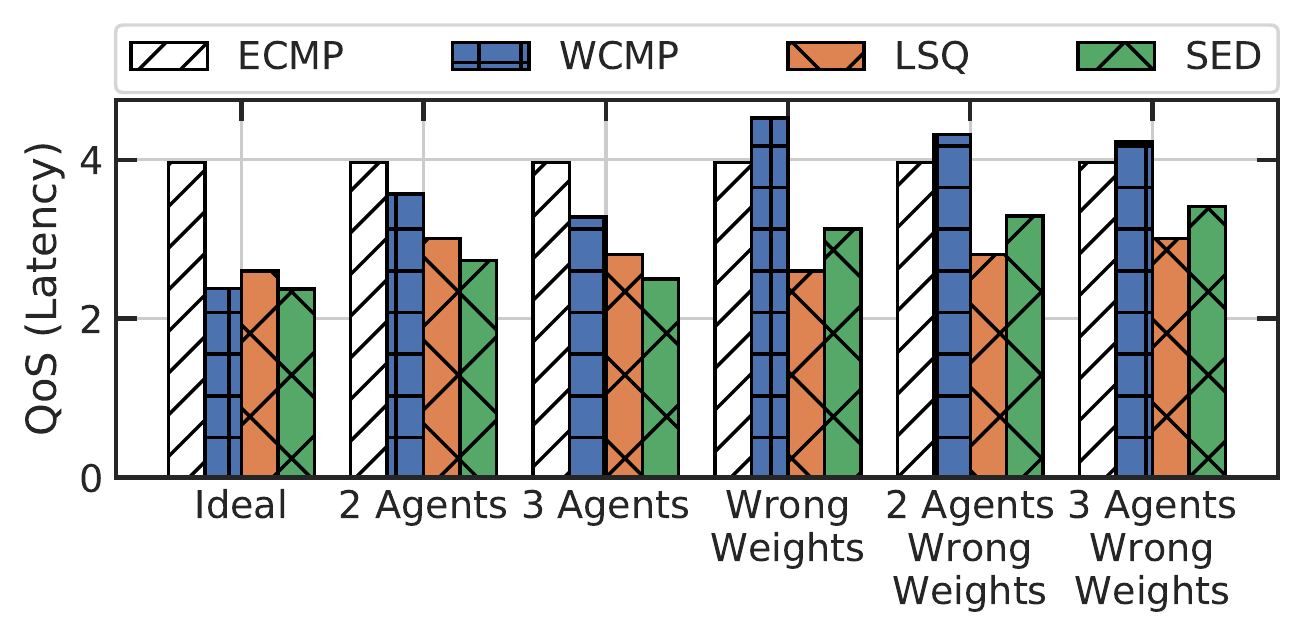}
		\caption{The performance of existing network load balancing algorithms degrades when observation becomes partial with multi-agents and weights are mis-configured.}
		\label{fig:motivation-inaccurate}
	\end{subfigure}
	
	\caption{Existing network load balancing algorithms are sub-optimal under real-world setups.}
	\label{fig:motivation}
	\vskip -.1in
\end{figure}

In this paper, we propose a distributed MARL-based load balancing algorithm that considers dynamically changing queue lengths (\eg sub-ms in modern DC networks~\cite{guo2015pingmesh}), and autonomously adapts to actual server processing capacities, with no additional communications among LB agents or servers.

\textbf{Markov Potential Games.} 

A potential game (PG)~\cite{monderer1996potential, sandholm2001potential, marden2009cooperative, candogan2011flows} has a special function called \textit{potential function}, which specifies a property that any individual deviation of the action for one player will change the value of its own and the potential function equivalently. A desirable property of PG is that pure NE always exists and coincides with the maximum of potential function in norm-form setting. Self-play~\cite{fudenberg1998theory} is provably converged for PG.
Markov games (MG) is an extension of normal-form game to a multi-step sequential setting. A combination of PG and MG yields the Markov potential games (MPG)~\cite{fox2022independent, leonardos2021global}, where pure NE is also proved to exist. Some algorithms~\cite{macua2018learning, mguni2021learning, fox2022independent} lying in the intersection of game theory and reinforcement learning are proposed for MPG.
For example, independent nature policy gradient is proved to converge to Nash equilibrium (NE) for MPG~\cite{fox2022independent}. 

\begin{table}[tbp]
\centering
\caption{Trade-offs among the probing frequency, measurement quality, and communication overhead.}
\label{tab:communication-overhead}
\begin{tabular}{cc|c|c|c|c|c}
\hline
\multicolumn{2}{c|}{Probing Frequency (/s)}                                                                                    & 2.22  & 2.86  & 4.00  & 6.67  & 20.00  \\ \hline
\multicolumn{1}{c|}{\multirow{2}{*}{RMSE}}                                                                      & CPU (\%)     & 48.33 & 44.56 & 39.84 & 32.65 & 21.97  \\
\multicolumn{1}{c|}{}                                                                                           & \#Job        & 2.07  & 1.85  & 1.61  & 1.31  & 0.91   \\ \hline
\multicolumn{1}{c|}{\multirow{2}{*}{Spearman's Corr.}}                                                          & CPU (\%)     & 0.28  & 0.40  & 0.52  & 0.68  & 0.85   \\
\multicolumn{1}{c|}{}                                                                                           & \#Job        & 0.47  & 0.56  & 0.66  & 0.77  & 0.89   \\ \hline
\multicolumn{1}{c|}{\multirow{2}{*}{\begin{tabular}[c]{@{}c@{}}Communication\\ Overhead (kbps)\end{tabular}}} & 2LB-7server  & 2.15  & 2.76  & 3.86  & 6.44  & 9.32   \\
\multicolumn{1}{c|}{}                                                                                           & 6LB-20server & 18.40 & 23.66 & 33.12 & 55.20 & 165.60 \\ \hline
\end{tabular}
\end{table}

\textbf{Multi-Agent RL}. MARL~\cite{yang2020overview} has been viewed as an important avenue for solving different types of games in recent years. For cooperative settings, a line of work based on joint-value factorisation have been proposed, involving VDN~\cite{sunehag2017value}, COMA~\cite{foerster2018counterfactual}, MADDPG~\cite{lowe2017multi}, and QMIX~\cite{rashid2018qmix}. For these works, a global reward is assigned to players within the team, but individual policies are optimised to execute individual actions, known as the CTDE setting.
MPG satisfies the assumptions of the value decomposition approach, with the well-specified potential function as the joint rewards. However, deploying CTDE RL models in real-world distributed system incurs additional communication latency and management overhead for synchronising agents and aggregating trajectories.
These additional management and communication overheads can incur substantial performance degradation -- constrained throughput and increased latency -- especially in data center networks.
As listed in Table~\ref{tab:communication-overhead}, when we use active probing to measure server utilisation information, higher probing frequencies give improved measurement quality--in terms of CPU usage and number of on-going jobs on the servers.
However, higher probing frequencies also incur increased communication overhead, especially in large-scale data center networks.
The detailed experimental setups, as well as both qualitative and quantitative analysis of the impact of communication overhead, are described in Sec.~\ref{app:results-ablation-comm}.
By leveraging the special structure of MPG, independent learning approach can be more efficient due to the decomposition of the joint state and action spaces, which is leveraged in the proposed methods. Methods like MATRPO~\cite{li2020multi}, IPPO~\cite{de2020independent} follow a fully decentralised setting, but for general cooperative games.

In terms of the distribution fairness, FEN~\cite{jiang2019learning} is proposed as a decentralised approach for fair reward distribution in multi-agent systems. They defined the fairness as coefficient of variation and decompose it for each individual agent. Another work~\cite{zimmer2021learning} proposes a decentralised learning method for fair policies in cooperative games. However, the decentralised learning manner in these methods are not well justified, while in this paper the load balancing problem is formally characterised as a MPG and the effectiveness of distributed training is verified.

\section{Methods}
\subsection{Problem Description}

We formulate the load balancing problem into a discrete-time dynamic game with strong distributed and concurrent settings, where no centralised control mechanism exists among agents.
We let $M$ denote the number of LB agents ($[M]$ denotes the set of LB agents $\{1, \dots, M\}$) and $N$ denote the number of servers ($[N]$ denotes the set of servers $\{1, \dots, N\}$).
At each time step (or round) $t \in H$ in a horizon $H$ of the game, each LB agent $i$ receives a workload $w_i(t) \in W$, where $W$ is the workload distribution, and the LB agent assigns a server to the task using its load balancing policy $\pi_i \in \Pi$, where $\Pi$ is the load balancing policy profile.
At each time-step $t$, a LB agent $i$ takes an action $a_i(t) = \{a_{ij}(t)\}_{j=1}^{N}$, according to which the tasks $w_i(t)$ are assigned with distribution $\alpha_i(t)$. $\alpha_{ij}(t)$ is the probability mass of assigning tasks to server $j$, $ \sum_{j=1}^{N} \alpha_{ij}(t) = 1$.
Therefore, at each time step, the workload assigned to server $j$ by the $i$-th LB is $w_i(t)\alpha_{ij}(t)$.
During each time interval, each server $j$ is capable of processing a certain amount of workload $v_j$ based on the property of each server (\eg provisioned resources including CPU, memory, \etc).
We have server load state (remaining workload to process) $X_j(T) = \sum_{t=0}^{T}\max\{0, \sum_{i=1}^M w_i(t)\alpha_{ij}(t) - v_j\} = \max\{0, \sum_{t=0}^{T}\sum_{i=1}^{M}w_i(t)\alpha_{ij}(t) - v_{j}T\} = \sum_{i=1}^{M}X_{ij}(T)$\footnote{$X_{ij}(T) = \sum_{t=0}^{T}\max\{0, w_i(t)\alpha_{ij}(t) - \frac{v_j}{M}\}$}.
Let $l_{j}$ denote the time for a server $j$ to process all remaining workloads, which is also the potential queuing time for new-coming tasks, $l_j(t) = \frac{X_j(t-1)+\sum_{i=1}^{M}w_i(t)\alpha_{ij}(t)}{v_j} = \frac{\sum_{i=1}^{M}X_{ij}(t-1)+w_i(t)\alpha_{ij}(t)}{v_j} = \sum_{i=1}^{M}l_{ij}(t)$.
Then transition from time step $t$ to time step $t+1$ is given in Alg.~\ref{alg:model-transition}.
Reward: $r_{i}(t) = R(\boldsymbol{l}(t), a_i(t), \delta_i(t))$, where $R$ is the reward function, $\boldsymbol{l}(t) = \sum_{j=1}^{N} l_j(t) = \sum_{i=1}^{M} l_{i}(t)$ denotes the estimated remaining time to process on each server, and $\delta_i(t)$ is a random variable that makes the process stochastic.
\begin{definition}(Makespan)
In the selfish load balancing problem, the makespan is defined as:
{\small
\begin{align}
 \text{MS} = \max_j(l_j), l_j = \sum_i l_{ij}
 \label{eq:makespan}
\end{align}}
\end{definition}
The network load balancing problem is featured as multi-commodity flow problems and is NP-hard, which makes it hard to solve with trivial algorithmic solution within micro-second level~\cite{sen2013scalable}.
This problem can be formulated as a constrained optimisation problem for minimizing the makespan over an horizon $t \in [H]$:

{\small
\begin{align}
    minimize \sum_{t=h}^{H}&\max_j l_j(t) \\[-5pt]
    s.t. \quad
    l_{j}(t)=&\frac{\sum_{i=1}^M (X_{ij}(t-1)+w_{i}(t)\alpha_{ij}(t))}{v_{j}}, \quad \sum_{i=1}^{M}w_i(t) \le \sum_{j=1}^{N}v_j,   \quad w_i, v_j\in(0, +\infty) \label{eq:lb_cons1}\\
    X_{ij}(T)&=\sum_{t=0}^{T} \max\{0, w_{i}(t)\alpha_{ij}(t) - \frac{v_j}{M}\},  \quad\sum_{j=1}^{N}\alpha_{ij}(t) = 1,  \quad \alpha_{ij} \in [0, 1] \label{eq:lb_cons2}
\end{align}}

In modern realistic network load balancing system, the arrival of network requests is usually unpredictable in both its arriving rate and the expected workload, which introduces large stochasticity into the problem. Moreover, due to the existence of noisy measurements and partial observations, the estimation of makespan can be inaccurate, which indicates the actual server load states or processing capacities are not correctly captured. Instant collisions of elephant workloads or bursts of mouse workloads often happen, which do not indicate server processing capacity thus misleading the observation. To solve this issue, we introduce \emph{fairness} as an alternative of the original objective makespan. Specifically, makespan is estimated on a per-server level, while the estimation of fairness can be decomposed to the LB level, which allows evaluating the individual LB performance without general loss. This is more natural in load balancing system due to the partial observability of LBs.

\subsection{Distribution Fairness}
\label{sec:fairness}
We mainly introduce two types of load balancing distribution fairness: (1) variance-based fairness (VBF) and (2) product-based fairness (PBF). It will be proved that optimization over either fairness will be sufficient but not necessary for minimising the makespan.

\begin{definition}(Variance-based Fairness)
\label{def:vbf}
For a vector of time to finish all remaining jobs $\boldsymbol{l}=[l_{1}, \dots, l_{N}]$ on each server $j\in[N]$, let $\overline{\boldsymbol{l}}(t) = \frac{1}{N}\sum_{j=1}^{N}\sum_{i=1}^{M}l_{ij}(t)$, the variance-based fairness for workload distribution is just the negative sample variance of the job time, which is defined as:
{\small
\begin{align}
    F(\boldsymbol{l}) = -\frac{1}{N}\sum_{j=1}^{N}\bigg(l_j(t)-\overline{\boldsymbol{l}}(t)\bigg)^2 = -\frac{1}{N}\sum_{j=1}^{N}l_j^2(t)+\overline{\boldsymbol{l}}^2(t).
\end{align}}
VBF defined per LB is: $F_i(\boldsymbol{l}_i) = -\frac{1}{N}\sum_{j=1}^{N}l_{ij}^2(t)+\overline{\boldsymbol{l}}_{i}^2(t)$, where $\overline{\boldsymbol{l}}_i(t) =  \frac{1}{N}\sum_{j=1}^{N}l_{ij}(t)$.
\end{definition}

\begin{lemma} The VBF for load balancing system satisfies the following property:
\label{lem:vbf}
{\small
\begin{align}
    F_i^{\pi_i, -\pi_i}(\boldsymbol{l}_i)-F_i^{\tilde{\pi}_i, -\pi_i}(\tilde{\boldsymbol{l}}_i) = F^{\pi_i, -\pi_i}(\boldsymbol{l})-F^{\tilde{\pi}_i, -{\pi}_i}(\tilde{\boldsymbol{l}})
\end{align}}
\end{lemma}
This property makes VBF a good choice for the reward function in load balancing tasks. We will see more discussions in later sections. Proof of the lemma is provided in Appendix~\ref{sec:app_vbf}.

\begin{proposition}
\label{prop:var_fairness}
Maximising the VBF is sufficient for minimising the makespan, subjective to the load balancing problem constraints (Eq.~\eqref{eq:lb_cons1} and \eqref{eq:lb_cons2}): $\max F(\boldsymbol{l}) \Rightarrow  \min \max_j(l_j)$. 
This also holds for per-LB VBF as $\max F_i(\boldsymbol{l}_i) \Rightarrow  \min \max_j(\boldsymbol{l}_i)$.
\end{proposition}

\begin{definition}(Product-based Fairness~\cite{yao2022reinforced}) For a vector of time to finish all remaining jobs $\boldsymbol{l}=[l_{1}, \dots, l_{N}]$ on each server $j\in[N]$, the product-based fairness for workload distribution is defined as: $F(\boldsymbol{l}) = F([l_1, \dots, l_N]) = \prod_{j \in [N]}\frac{l_j}{\max(\boldsymbol{l})}$.
PBF defined per LB is: $    F_i(\boldsymbol{l}_i) = F([l_{i1}, \dots, l_{iN}]) = \prod_{j \in [N]}\frac{l_{ij}}{\max(\boldsymbol{l}_i)}$.
\end{definition}

\begin{proposition}
\label{prop:pro_fairness}
Maximising the product-based fairness is sufficient for minimising the makespan, subjective to the load balancing problem constraints (Eq.~\eqref{eq:lb_cons1} and \eqref{eq:lb_cons2}): $\max F(\boldsymbol{l}) \Rightarrow  \min \max(\boldsymbol{l})$.
\end{proposition}
Proofs of proposition \ref{prop:var_fairness} and \ref{prop:pro_fairness} are in Appendix~\ref{sec:app_vbf} and\ref{sec:app_pbf}, respectively. From proposition~\ref{prop:var_fairness} and \ref{prop:pro_fairness}, we know that the two types of fairness can serve as an effective alternative objective for optimising the makespan, which will be leveraged in our proposed MARL method as valid reward functions.

\subsection{Game Theory Framework}

Markov game is defined as $\mathcal{MG}(H, M, \mathcal{S}, \mathcal{A}_{\times M}, \mathbb{P}, r_{\times M})$, where $H$ is the horizon of the game, $M$ is the number of player in the game, $\mathcal{S}$ is the state space, $\mathcal{A}_{\times M}$ is the joint action space of all players, $\mathcal{A}_i$ is the action space of player $i$, $\mathbb{P}=\{\mathbb{P}_h\}, h\in[H]$ is a collection of transition probability matrices $\mathbb{P}_h: \mathcal{S}\times \mathcal{A}_{\times M} \rightarrow \Pr(\mathcal{S})$, $r_{\times M}=\{r_i|i\in[M]\}, r_i:\mathcal{S}\times\mathcal{A}_{\times M}\rightarrow \mathbb{R}$ is the reward function for $i$-th player given the joint actions.  The stochastic policy space for the $i$-th player in $\mathcal{MG}$ is defined as $\Pi_i: \mathcal{S} \rightarrow \Pr(\mathcal{A}_i)$, $\Pi=\{\Pi_i\}, i\in[M]$.

For the Markov game $\mathcal{MG}$, the state value function $V_{i, h}^{\boldsymbol{\pi}}: \mathcal{S}\rightarrow \mathbb{R}$ and state-action value function $Q_{i, h}^{\boldsymbol{\pi}}: \mathcal{S}\times \mathcal{A}\rightarrow \mathbb{R}$ for the $i$-th player at step $h$ under policy $\boldsymbol{\pi}\in\Pi_{\times M}$ is defined as:
\begin{equation} \label{eq:V_value}
{\small\begin{aligned}
	 V_{i, h}^{\boldsymbol{\pi}}(s):= \mathbb{E}_{\boldsymbol{\pi},\mathbb{P}}\bigg[\sum_{h' =
        h}^H r_{i, h'}(s_{h'}, \boldsymbol{a}_{h'}) \bigg| s_h = s\bigg], 	 Q_{i, h}^{\boldsymbol{\pi}}(s, \boldsymbol{a}):= \mathbb{E}_{\boldsymbol{\pi},\mathbb{P}}\bigg[\sum_{h' =
        h}^H r_{i, h'}(s_{h'}, \boldsymbol{a}_{h'}) \bigg| s_h = s, a_h=\boldsymbol{a}\bigg].
\end{aligned}}
\end{equation}

\begin{definition}($\epsilon$-approximate Nash equilibrium) Given a Markov game $\mathcal{MG}(H, M, \mathcal{S}, \mathcal{A}_{\times M}, \mathbb{P}, \Pi_{\times M}, r_{\times M})$, let $\pi_{-i}$ be the policies of the players except for the $i$-th player, the policies $(\pi_i^*, \pi_{-i}^*)$ is an $\epsilon$-Nash equilibrium if $\forall i\in[M], \exists \epsilon>0$, 
{\small
\begin{align}
V_i^{\pi^*_i, \pi^*_{-i}}(s)\ge V_i^{\pi_i, \pi^*_{-i}}(s)-\epsilon, \forall \pi_i \in \Pi_i.
\end{align}}
If $\epsilon=0$, it is an exact Nash equilibrium.
\end{definition}

\begin{definition}(Markov Potential Game) A Markov game $\mathcal{M}(H, M, \mathcal{S}, \mathcal{A}_{\times M}, \mathbb{P}, \Pi_{\times M}, r_{\times M})$ is a Markov potential game (MPG) if $\forall i\in[M], \pi_i, \tilde{\pi}_i\in\Pi_i, \pi_{-i}\in\Pi_{-i}, s\in\mathcal{S}$, 
{\small
\begin{align}
V_i^{\pi_i, \pi_{-i}}(s) - V_i^{\tilde{\pi}_i, \pi_{-i}}(s)=\phi^{\pi_i, \pi_{-i}}(s)-\phi^{\tilde{\pi}_i, \pi_{-i}}(s),
\end{align}}
where $\phi(\cdot)$ is the potential function independent of the player index. 
\end{definition}

\begin{lemma} Pure NE (PNE) always exists for PG, local maximisers of potential function are PNE. PNE also exists for MPG. \cite{monderer1996potential}
\label{lem:pne}
\end{lemma}

\begin{theorem}
\label{thm:mpg_vbf}
Multi-agent load balancing is MPG with the VBF $F_i(\boldsymbol{l}_i)$ as the reward $r_i$ for each LB agent $i\in[M]$, then suppose for $\forall s\in \mathcal{S}$ at step $h\in[H]$, the potential function is time-cumulative total fairness: $\phi^{\pi_i, -\pi_i}(s)=\sum_{t =h}^H F^{\pi_i, -\pi_i}(\boldsymbol{l}(t))$.
\end{theorem}
The proof of the theorem is based on Lemma~\ref{lem:vbf}, and it's provided in Appendix~\ref{sec:app_vbf_mpg}.
This theorem is essential for establishing our method, since it proves that multi-agent load balancing problem can be formulated as a MPG with the time-cumulative VBF as its potential function. Also, the choice of per-LB VBF as reward function for individual agent is critical for making it MPG, it is easy to verify that PBF cannot guarantee such property. From Lemma~\ref{lem:pne} we know the maximiser of potential function is the NE of MPG, and from proposition~\ref{prop:var_fairness} it is known that maximising the VBF gives the sufficient condition for minimising the makespan. Therefore, an effective independent optimisation with respect to the individual reward function specified in the above theorem will lead the minimiser of makespan for load balancing tasks. The effective independent optimisation here means the NE of MPG is achieved.

\subsection{Distributed LB Method}
\label{sec:rl}

\begin{figure}[t]
	\centering
	\centerline{\includegraphics[width=.9\columnwidth]{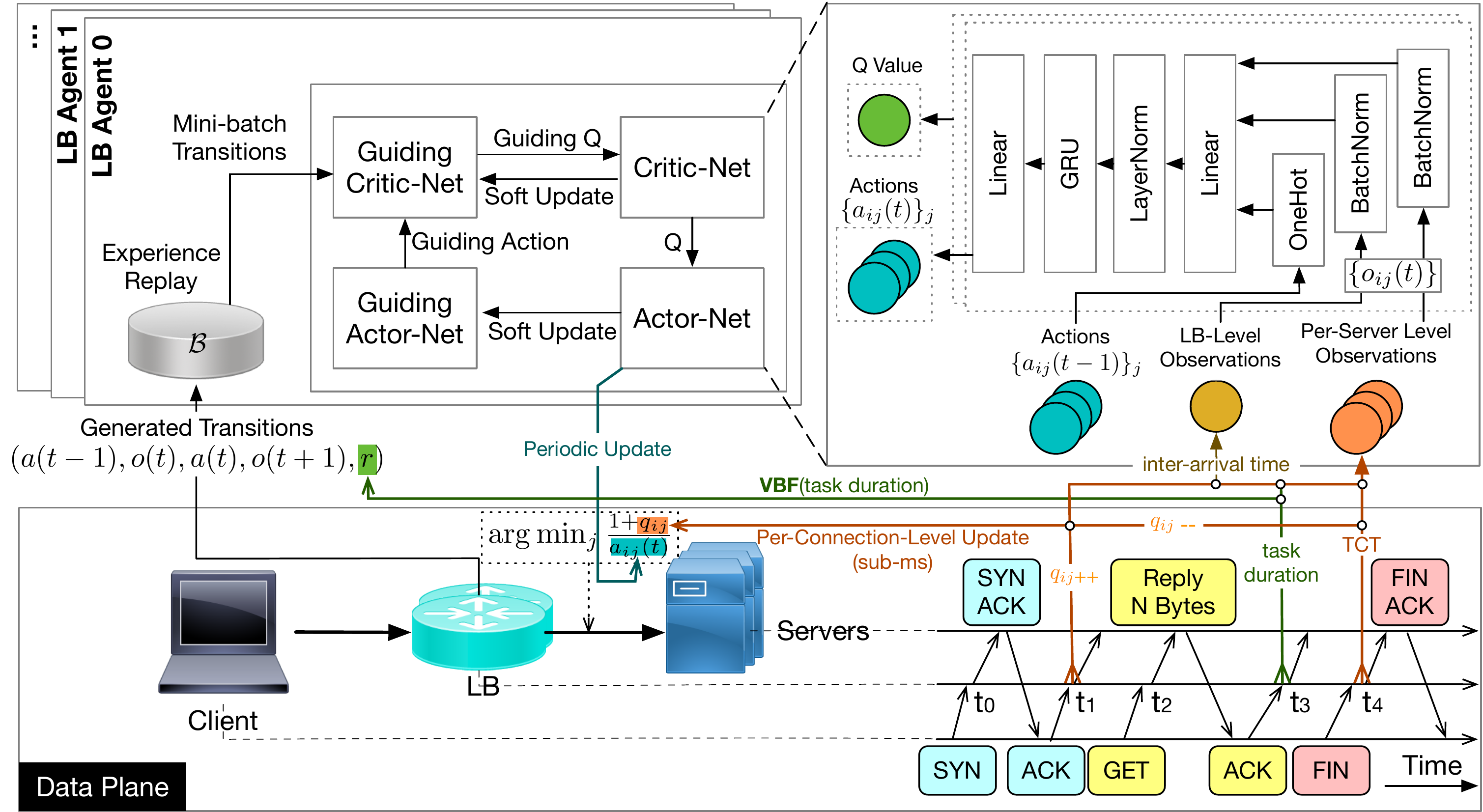}}
	\caption{Overview of the proposed distributed MARL framework for network LB. }
	\label{fig:design}
\end{figure}

With the above analysis, the load balancing problem can be formulated as an episodic version of multi-player partially observable Markov game, which we denote as $\mathcal{POMG}(H, M, \mathcal{S}, \mathcal{O}_{\times M}, \mathbb{O}_{\times M}, \mathcal{A}_{\times M}, \mathbb{P}, r_{\times M})$, where $M, H, \mathcal{S},\mathcal{A}_{\times M}$ and $\mathbb{P}$ follow the same definitions as in Markov game $\mathcal{MG}$, $\mathcal{O}_{\times M}$ contains the observation space $O_i$ for each player, $\mathbb{O}=\{\mathbb{O}_h\}, h\in[H]$ is a collection of observation emission matrices, $\mathbb{O}_{i, h}: \mathcal{S} \rightarrow \Pr(\mathcal{O}_{i})$, $r_{\times M}=\{r_i|i\in[M]\}, r_i:\mathcal{O}_i\times\mathcal{A}_{\times M}\rightarrow \mathbb{R}$ is the reward function for $i$-th LB agent given the joint actions. The stochastic policy space for the $i$-th agent in $\mathcal{POMG}$ is defined as $\Pi_i: \mathcal{O}_i \rightarrow \Pr(\mathcal{A}_i)$. As discussed in Sec.~\ref{sec:background}, the partial observability comes from the fundamental configuration of network LBs in DC networks, which allows LBs to observe only a partial of network traffic and does not give LBs information about the tasks (\eg expected workload) distributed from each LB. The reward functions in our experiments are variants of distribution fairness introduced in Sec.~\ref{sec:fairness}. The potential functions can be defined accordingly based on the two fairness indices. The overview of the proposed distributed MARL framework is shown in Fig.~\ref{fig:design}.

In MPG, independent policy gradient allows finding the maximum of the potential function, which is the PNE for the game. This inspires us to leverage the policy optimisation in a decomposed manner, \emph{i.e.}, distributed RL for policy learning of each LB agent. However, due to the partial observability of the system and the challenge of directly estimating the makespan (Eq.~\eqref{eq:makespan}), each agent cannot have a direct access to the global potential function. To address this problem, the aforementioned fairness (Sec.~\ref{sec:fairness}) can be deployed as the reward function for each agent, which makes the value function as a valid alternative for the potential function as an objective. This also transforms the joint objective (makespan or potential) to individual objectives (per LB fairness) for each agent. Proposition~\ref{prop:var_fairness} and \ref{prop:pro_fairness} verify that optimising towards these fairness indices is sufficient for minimising the makespan.

\begin{algorithm}[tbp]
\footnotesize
\caption{Distributed LB for MPG}
\label{alg:dec_lb_mpg}
\begin{algorithmic}[1]
\STATE \textbf{Initialise:}
\STATE \quad LB policy $\pi_{\theta_i}$ and critic $Q_{\phi_i}$ networks, replay buffer $\mathcal{B}_i, \forall i \in [M]$;
\STATE \quad server processing speed function $v_j, \forall j \in [N]$;
\STATE \quad initial observed instant queue length on server $j$ by the $i$-th LB: $q_{ij}=0,  \forall i\in[M], j\in[N]$.
\WHILE {not converge}
\STATE Reset server load state $X_j(1) \gets 0, \forall j\in[N]$
\STATE Each LB agent $i$ ($i\in[M]$) receives individual observation ${o}_i(1)$
\FOR {$t=1,\dots, H$}
    \STATE Initialise distributed workload $m_{ij}, w_{i}(t) \gets 0, i\in[M], j\in[N]$ 
    \STATE Get actions $a_{i}(t) \gets \{a_{ij}(t)\}_{j=1}^N = \pi_{\theta_i}({o}_{i}(t)), i\in[M]$
    \FOR {job $\tilde{w}$ arrived at LB $i$ between timestep [$t$, $t+1$)}
        \STATE LB $i$ assigns $\tilde{w}$ to server $j = \arg \min_{k\in[N]}\frac{q_{ik}(t)+1}{a_{ik}(t)}$ \alglinelabel{line:job_update_start}
        \STATE $m_{ij} \gets m_{ij}+\tilde{w}$, $w_{i}(t)\gets w_{i}(t)+ \tilde{w}$
        \STATE $\alpha_{ij}(t) \gets \frac{m_{ij}}{w_{i}(t)}$ \alglinelabel{line:job_update_end}
    \ENDFOR
    
    \FOR {each server $j$}
        \STATE Update workload: {$X_{ij}(t+1) \gets \max\{X_{ij}(t)+w_{i}(t)\alpha_{ij}(t) - \frac{v_j}{M}, 0\}$}
        \STATE $X_{j}(t+1) \gets \sum_{i=1}^{M}X_{ij}(t)$
    \ENDFOR
    \STATE Each agent receives individual reward $r_i(t)$
	\STATE Each agent $i$ collects observation ${o}_i(t+1), i\in[M]$
    \STATE Update replay buffer: $\mathcal{B}_i=\mathcal{B}_i\bigcup (a_i(t-1), {o}_i(t), a_i(t), r_i(t), {o}_i(t+1)), i\in[M]$
\ENDFOR
\STATE Update critics with gradients: $\nabla_{\phi_i}\mathbb{E}_{(o_i,a_i, r_i, o'_i)\sim \mathcal{B}_i}\bigg[\bigg(Q_{\phi_i}(o_i,a_i)-r_i-\gamma V_{\tilde{\phi}_i}(o_i^\prime)\bigg)^2\bigg]$\\
\STATE where $V_{\tilde{\phi}_i}(o_i^\prime)=\mathbb{E}_{ (o_i^\prime, a'_i)\sim\mathcal{B}_i}[Q_{\tilde{\phi}_i}(o_i^\prime, a_i^\prime)-\alpha\log\pi_{\theta_i}(a_i^\prime|o_i^\prime)], i\in[M]$
\STATE Update policies with gradients: -$\nabla_{\theta_i}\mathbb{E}_{o_i\sim\mathcal{B}_i}[\mathbb{E}_{a\sim\pi_{\theta_i}}[\alpha\log\pi_{\theta_i}(a_i|o_i)-Q_{\phi_i}(o_i,a_i)]], i\in[M]$
\ENDWHILE
\RETURN final models of learning agents
\end{algorithmic}
\end{algorithm}

Alg.~\ref{alg:dec_lb_mpg} shows the proposed distributed LB for load balancing problem, which is a partially observable MPG. The distributed policy optimisation is based on Soft Actor-Critic (SAC)~\cite{haarnoja2018soft} algorithm, which is a type of maximum-entropy RL method. It optimises the objective $\mathbb{E}[\sum_t \gamma^t r_t+\alpha \mathcal{H(\pi_\theta)}]$, whereas $\mathcal{H}(\cdot)$ is the entropy of the policy $\pi_\theta$. Specifically, the critic $Q$ network is updated with gradient $\nabla_\phi\mathbb{E}_{o,a}\bigg[\bigg(Q_\phi(o,a)-r(o,a)-\gamma \mathbb{E}_{o^\prime}[V_{\tilde{\phi}}(o^\prime)]\bigg)^2\bigg]$, where $V_{\tilde{\phi}}(o^\prime)=\mathbb{E}_{a^\prime}[Q_{\tilde{\phi}}(o^\prime, a^\prime)-\alpha\log\pi_\theta(a^\prime|o^\prime)]$ and $Q_{\tilde{\phi}}$ is the target $Q$ network; the actor policy $\pi_\theta$ is updated with the gradient $\nabla_\theta\mathbb{E}_o[\mathbb{E}_{a\sim\pi_\theta}[\alpha\log\pi_\theta(a|o)-Q_\phi(o,a)]]$. Other key elements of RL methods involve the observation, action and reward function, which are detailed as following.

\textbf{Observation.} 
Each LB agent partially observes over the traffic that traverses through itself, including per-server-level and LB-level measurements. For each LB, per-server-level observations consist of -- for each server -- the number of on going tasks, and sampled task duration and task completion time (TCT).
Specifically, in Alg.~\ref{alg:dec_lb_mpg} line \ref{line:job_update_start}-\ref{line:job_update_end}, $w_{i}$ is the coming workload on servers assigned by $i$-th LB, and it is not observable for LB. $q_{ik}+1$ is the locally observed number of tasks on $k$-th server by $i$-th LB, due to the real-world constraints of limited observability at the Transport layer. The ``+1'' is for taking into account the new-coming task. 
Observations of task duration and TCT samples, along with LB-level measurements which sample the task inter-arrival time as an indication of overall system load state, are reduced to 5 scalars -- \ie average, 90th-percentile, standard deviation, discounted average and weighted discounted average\footnote{Discounted average weights are computed as $0.9^{t^\prime-t}$, where $t$ is the sample timestamp and $t^\prime$ is the moment of calculating the reduced scalar.} -- as inputs for LB agents.

\textbf{Action.}
To bridge the different timing constraints between the control plane and data plane, each LB agent assigns the $j$-th server to newly arrived tasks using the ratio of two factors, $\arg \min_{k\in[N]} \frac{q_{ik}+1}{a_{ik}}$, where the number of on-going tasks $q_{ik}$ helps track dynamic system server occupation at per-connection level -- which allows making load balancing decision at $\mu$s-level speed -- and $a_{ik}$ is the periodically updated RL-inferred server processing speed. As in line \ref{line:job_update_end} of Alg.~\ref{alg:dec_lb_mpg}, $\alpha_{ij}(t)$ is a statistical estimation of workload assignment distribution at time interval $[t, t+1)$.

\textbf{Reward.} The individual reward for distributed MPG LB is chosen as the VBF (as Def.~\ref{def:vbf}) of the discounted average of sampled task duration measured on each LB agent, such that the LB group jointly optimise towards the potential function defined in Eq.~\eqref{thm:mpg_vbf}.
Task duration information is gathered as the time interval between the end of connection initialisation (\eg $3$-way handshake for TCP traffic) and the acknowledgement to the first data packet (\eg the first ACK packet for TCP traffic).
Given the limited and partial observability of LB agents, task duration information approximates the remaining workload $\boldsymbol{l}$ by measuring the queuing and processing delay for new-coming tasks on each server.
This PBF- and MS-based rewards are also implemented for CTDE MARL algorithm as a comparison.

\textbf{Model.} The architecture of the proposed RL framework is depicted in Fig.~\ref{fig:design}.
Each LB agent consists of a replay buffer, and a pair of actor-critic networks, whose architecture is depicted on the top right.
There is also a pair of guiding actor-critic networks, with the same network architectures but updated in a delayed and soft manner.
Each LB agent takes observations $o_i(t)$ extracted from the data plane (\eg numbers of ongoing tasks $\{q_{ij}\}$, task duration, TCT) and actions from previous timestep $a_{i}(t-1)$ as inputs, and periodically generates new actions $a_{i}(t)$, which is used to update the server assignment function $\arg \min_{j\in[N]} \frac{q_{ij}+1}{a_{ij}}$ in the data plane. The gated recurrent units (GRU)~\cite{chung2014empirical} are applied for all agents to leverage the sequential history information for handling partial observability. 

\section{Evaluation}
\label{sec:experiment}

We developed (i) an event-based simulator (App.~\ref{app:implement-simulator}) to study the distance between the NE achieved by the proposed algorithm and the NE achieved by the theoretical optimal load balancing policy (with perfect observation), and (ii) a realistic testbed (App.~\ref{app:implement-testbed}) on physical servers in a DC network providing Apache web services, with real-world network traffic~\cite{wiki_traces}, to evaluate the real-world performance of the proposed algorithm, in comparison with in-production state-of-the-art LB~\cite{maglev}.

\begin{figure}[tbp]
	\centering
	\begin{subfigure}{0.4\columnwidth}
		\centering
		\includegraphics[height=1.3in]{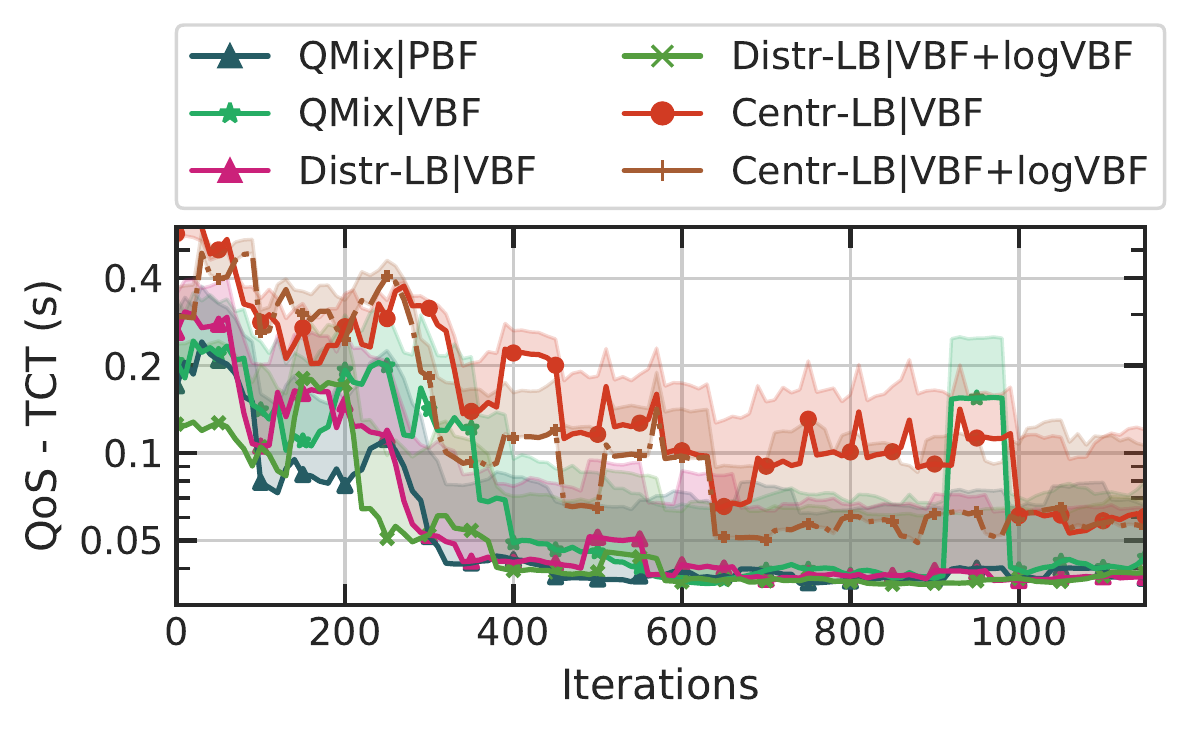}
		\vskip -.1in
		\caption{Learning curves.}
		\label{fig:eval-train}
	\end{subfigure}
	\hspace{.05in}
	\begin{subfigure}{0.55\columnwidth}
		\centering
		\includegraphics[height=1.3in]{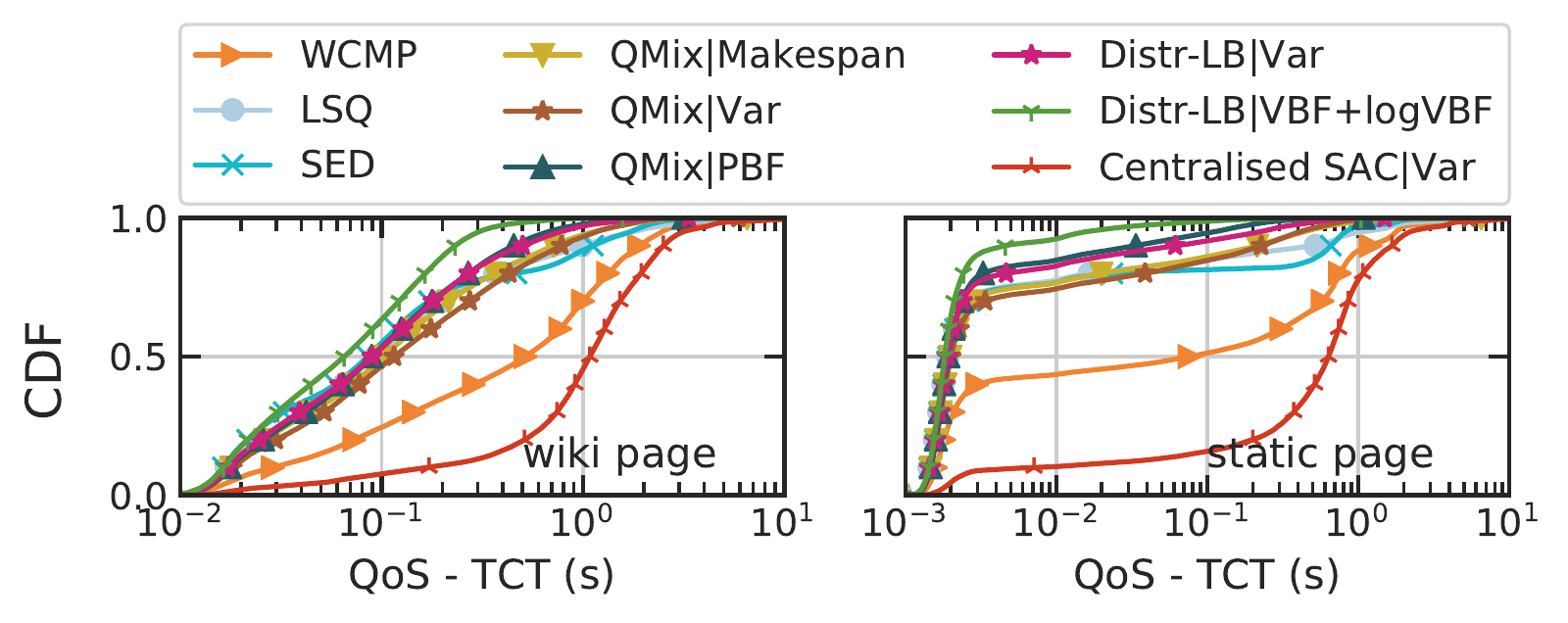}
		\vskip -.1in
		\caption{CDF of TCT.}
		\label{fig:eval-cdf}
	\end{subfigure}
	\caption{Experimental results show that the proposed distributed RL framework using proposed VBF as rewards converges and effectively achieves better load balancing performance (lower TCT and better QoS) than existing LB algorithms and CTDE RL algorithms.}
	\label{fig:eval}
\end{figure}

\begin{table}[tbp]
    \scriptsize
    \centering
    \caption{Comparison of average QoS (s) in moderate-scale real-world network setup.}
    \begin{tabular}{c|c|c|c|c|c}
    \toprule
    \multicolumn{2}{c}{\multirow{2}{*}{Method}} &  \multicolumn{2}{|c|}{Period III ($758.787$ queries/s)} & \multicolumn{2}{c}{Period IV ($784.522$ queries/s)} \\
    \cline{3-6}
     \multicolumn{2}{c|}{} & \multicolumn{1}{c|}{Wiki} & \multicolumn{1}{c|}{Static} & \multicolumn{1}{c|}{Wiki} & \multicolumn{1}{c}{Static} \\
    \hline
     \multicolumn{2}{c|}{WCMP} &  $0.412\pm0.101$ & $0.134\pm0.059$ & $ 0.834\pm0.323$ & $0.492\pm0.276 $ \\ 
      \multicolumn{2}{c|}{LSQ} & $0.620\pm0.442$ & $0.339\pm0.316$ & $0.357\pm0.373 $ & $0.173\pm0.299 $  \\ 
      \multicolumn{2}{c|}{SED} &  $0.215\pm0.210$ & $0.051\pm0.081$ & $0.346\pm0.496$ &$0.169\pm0.330 $ \\ \cline{1-2}
      \multirow{2}{*}{\textbf{RLB-SAC}~\cite{yao2022reinforced}} 
      & Jain & $0.193\pm0.073$ & $0.026\pm0.022$ & $0.204\pm0.084$ & $0.039\pm0.047$ \\
      & G    & $0.149\pm0.049$ & $0.015\pm0.011$ & $0.155\pm0.052$ & $0.011\pm0.011$ \\	\cline{1-2}
      \multirow{3}{*}{\textbf{QMix-LB}} & MS &  $0.217\pm0.157$ & $0.048\pm0.069$ &$0.263\pm0.202 $ &$0.073\pm0.092 $  \\ 
    & VBF &  $0.141\pm0.025$ & $0.008\pm0.004$ &$ 0.286\pm0.162$ & $0.068\pm0.066 $ \\ 
      & PBF &  $0.211\pm0.153$ & $0.047\pm0.078$ & $0.181\pm0.042 $& $0.018\pm0.009 $ \\ \cline{1-2}
    \multirow{2}{*}{\makecell{\textbf{Distr-LB}\\(this paper)}} & VBF &  $0.159\pm0.054$ & $0.017\pm0.009$ & $0.196\pm0.091 $& $ 0.032\pm0.033$\\ 
    & VBF+$\log$VBF &  $\mathbf{0.108\pm0.022}$ & $\mathbf{0.004\pm0.001}$ & $\mathbf{0.104\pm0.013}$ & $\mathbf{0.006\pm0.003} $\\ \cline{1-2}
    \multirow{2}{*}{\textbf{Centr-LB}} & VBF & $1.068\pm0.386$ & $0.570\pm0.378$ & $ 1.378\pm0.377$& $ 0.867\pm0.350$\\ 
    & VBF+$\log$VBF & $0.759\pm0.254$ & $0.306\pm0.222$ & $ 1.013\pm0.168$&$ 0.520\pm0.167$ \\
    \bottomrule
    \end{tabular}
    \label{tab:compare_small_scale_518}
\end{table}

\textbf{Moderate-Scale Real-World Testbed:} As depicted in Fig.~\ref{fig:eval-train}, in a moderate-scale real-world DC network setup with $2$ LB agents and $7$ servers, after $120$ episodes of training, the proposed distributed LB (Distr-LB) algorithm is able to learn from the environment based on VBF as rewards, and it converges to offer better QoS than QMix.
Centralised RL agent (Centr-LB) has difficulties to learn within $120$ episodes because of the increased state and action space.
An empirical finding is that, by adding a log term to the VBF-based reward for Distr-LB, we help LB agents to become more sensitive to close-to-$0$ VBF during training ($\nabla_x \log f(x)>\nabla_xf(x) \text{ when } f(x)< 1$), therefore achieving better load balancing performance.
As depicted in Fig.~\ref{fig:eval-cdf}, when comparing with in-production LB algorithms (WCMP, LSQ, SED), Distr-LB shows clear performance gains and reduced TCT for both types of web pages -- Wikipedia pages require to query SQL databases thus they are more CPU-intensive, while static pages are IO-intensive.
The comparison of average TCT using different LB algorithms is shown in Table~\ref{tab:compare_small_scale_518} ($99$th percentile TCT in Table~\ref{tab:compare_small_scale_full_99}).
The proposed Distr-LB also shows superior performance than the RL-based solution (RLB-SAC)~\cite{yao2022reinforced} because of (i) a well designed MARL framework, and (ii) the use of recurrent neural network to handle load balancing problem as a sequential problem.

\begin{table}[tbp]
\scriptsize
\centering
\caption{Comparison of average QoS (s) in moderate-scale simulator for different types of applications.}
\begin{tabular}{c|c|c|c|c}
\toprule
\multicolumn{2}{c|}{} & \multicolumn{1}{c|}{50\%-CPU+50\%-IO} &  \multicolumn{1}{c|}{75\%-CPU+25\%-IO} & \multicolumn{1}{c}{100\%-CPU} \\ \hline
\multicolumn{2}{c|}{Oracle} & $6.437\pm1.006 $ & $1.469\pm0.102 $ & $1.291\pm0.075 $ \\ \cline{1-2}
\multirow{2}{*}{\textbf{QMix-LB}} & PBF & $ 10.230\pm0.108$ & $1.828\pm0.054 $ & $ 2.200\pm0.288$ \\ 
 & VBF & $10.936\pm0.470 $  & $2.023\pm0.255 $ & $2.125\pm0.074 $ \\ \cline{1-2}
\multirow{2}{*}{\makecell{\textbf{Distr-LB}\\(this paper)}} & VBF& $ 10.335\pm0.362$ &$\mathbf{1.695\pm0.104} $  & $ \mathbf{1.643\pm0.016}$ \\ 
& VBF+$\log$VBF & $ \mathbf{8.797\pm0.459}$ &$1.873\pm0.328 $  & $ 2.004\pm0.042$ \\
\bottomrule
\end{tabular}
\label{tab:simulation-optimal-distance}
\end{table}

\begin{table}[tbp]
    \vskip -.2in
    \scriptsize
    \centering
    \caption{Comparison of average QoS (s) in large-scale real-world network setup.}
    \begin{tabular}{c|c|c|c|c|c}
    \toprule
    \multicolumn{2}{c}{\multirow{2}{*}{Method}} &  \multicolumn{2}{|c|}{Period I ($2022.855$ queries/s)} & \multicolumn{2}{c}{Period II ($2071.129$ queries/s)} \\
    \cline{3-6}
    \multicolumn{2}{c|}{} & \multicolumn{1}{c|}{Wiki} & \multicolumn{1}{c|}{Static} & \multicolumn{1}{c|}{Wiki} & \multicolumn{1}{c}{Static} \\
    \hline
     \multicolumn{2}{c|}{WCMP}  &        $0.473\pm0.102$ & $0.194\pm0.090$ & $0.460\pm0.241$ & $0.239\pm0.212$ \\ 
      \multicolumn{2}{c|}{LSQ}  &        $0.266\pm0.127$ & $0.063\pm0.065$ & $0.218\pm0.246$ & $0.082\pm0.152$  \\ 
      \multicolumn{2}{c|}{SED}  &        $0.169\pm0.062$ & $0.020\pm0.025$ & $0.166\pm0.141$ & $0.050\pm0.070$ \\
      \multicolumn{2}{c|}{RLB-SAC-G\cite{yao2022reinforced}}  &  $0.182\pm0.049$ & $0.013\pm0.009$ & $0.111\pm0.029$ & $0.010\pm0.009$ \\ \cline{1-2}
    \multirow{2}{*}{\textbf{QMix-LB}} 
    & VBF                       &  $0.181\pm0.062$ & $0.019\pm0.020$ & $0.188\pm0.147$ & $0.052\pm0.075$ \\ 
    & PBF                       &  $0.210\pm0.041$ & $0.013\pm0.006$ & $0.104\pm0.009$ & $0.005\pm0.003 $ \\ \cline{1-2}
    \multirow{2}{*}{\makecell{\textbf{Distr-LB}\\(this paper)}} 
    & VBF                       & $0.228\pm0.055$ & $0.019\pm0.011$ & $0.174\pm0.102$& $ 0.035\pm0.039$\\ 
    & VBF+$\log$VBF             & $\mathbf{0.161\pm0.033}$ & $\mathbf{0.008\pm0.003}$ & $\mathbf{0.094\pm0.015}$ & $\mathbf{0.004\pm0.001}$\\ 
    \bottomrule
    \end{tabular}
    \label{tab:compare_large_scale}
\end{table}

\textbf{NE Gap Evaluation with Simulation:} To evaluate the gap between the performance of Distr-LB and the theoretical optimal policy, we implement in the simulator an Oracle LB, which has perfect observation (inaccessible in real world) over the system and minimises makespan for each load balancing decision. 
Table~\ref{tab:simulation-optimal-distance} shows that, for different types of applications, Distr-LB is able to achieve closer-to-optimal performance than QMix.
As the simulator is implemented based on the load balancing model formulated in this paper, our theoretical analysis can be directly applied, and VBF -- as a potential function -- helps independent cooperative LB agents to achieve good performance.
The additional $log$ term shows empirical performance gains in real-world system, yet it is not necessarily the case in these simulation results.
On one hand, the generated traffic of tasks in the simulation has higher expected workload ($>1$s mean and stddev), while the $log$ terms is more sensitive to close-to-$0$ variances, which is the case in real-world experimental setups.
On the other hand, though the simulator models the formulated LB problem, it fails to captures the complexity in the real-world system -- \eg Apache backlog, multi-processing optimisation, context switching, multi-level cache, network queues \etc
For instance, batch processing~\cite{vpp} helps reduce cache and instruction misses, yet yields similar processing time for different tasks, thus the variance of task processing delay decreases and becomes closer to $0$ in real-world system. 
The additional $log$ term exaggerates the low variance differences to better evaluate load balancing decisions. 
More detailed description about the simulator implementation can be found in App~\ref{app:implement-simulator} and ablation study on reward engineering is presented in App~\ref{app:results-reward-engineering}.

\begin{table}[tbp]
\scriptsize
\centering
\caption{Comparison of $99$-th percentile QoS (s) of Wiki pages under different traffic rates using large-scale real-world setup.}
\begin{tabularx}{\textwidth}{c|c|X|X|X|X|X|X|X|X|X}
\toprule
 \multicolumn{2}{c|}{\multirow{2}{*}{Method}} & \multicolumn{9}{c}{Traffic Rate (queries/s)}  \\ \cline{3-11}
 \multicolumn{2}{c|}{} & 731.534 & 1097.3  & 1463.067 & 1828.834  & 2194.601 & 2377.484  & 2560.368 & 2743.251  & 2926.135\\ \hline
 \multicolumn{2}{c|}{\multirow{2}{*}{LSQ}} & 0.175\newline$\pm$0.015  & 0.212\newline$\pm$0.025  & 0.249\newline$\pm$0.043 & 0.342\newline$\pm$0.121  & 0.827\newline$\pm$0.572  & 2.103\newline$\pm$0.654 & 10.662\newline$\pm$2.557  & 17.656\newline$\pm$0.714 & 17.999\newline$\pm$0.253\\ \hline
 \multicolumn{2}{c|}{\multirow{2}{*}{SED}}  & 0.201\newline$\pm$0.022  & 0.261\newline$\pm$0.079  & 0.322\newline$\pm$0.099 & 0.360\newline$\pm$0.088  & 0.618\newline$\pm$0.268  & 2.175\newline$\pm$1.328 &  11.444\newline$\pm$3.861 & 22.086\newline$\pm$4.892 & 22.727\newline$\pm$5.632 \\ \hline
 \multirow{4}{*}{\makecell{\textbf{Distr-LB}\\(this paper)}}  &  \multirow{2}{*}{VBF} & \textbf{0.160\newline$\pm$0.010}  & \textbf{0.205\newline$\pm$0.036} & \textbf{0.248\newline$\pm$0.086}  & \textbf{0.284\newline$\pm$0.113}  & 0.567\newline$\pm$0.306  & \textbf{1.276\newline$\pm$0.647} & 7.005\newline$\pm$1.147 & 10.560\newline$\pm$1.042 & 15.745\newline$\pm$0.254 \\ \cline{2-11}
 & \multirow{2}{*}{VBF+$\log$VBF}  &  0.161\newline$\pm$0.008 & 0.216\newline$\pm$0.052  & 0.249\newline$\pm$0.068 & 0.348\newline$\pm$0.122  & \textbf{0.439\newline$\pm$0.121}  & 1.533\newline$\pm$0.670  &  \textbf{4.427\newline$\pm$0.443} & \textbf{9.391\newline$\pm$0.329} & \textbf{15.347\newline$\pm$0.572} \\ 
\bottomrule
\end{tabularx}
\vskip -.1in
\label{tab:99_qos_wiki}
\end{table}

\begin{table}[tbp]
\scriptsize
\centering
\caption{Comparison of $99$-th percentile QoS (s) of static pages under different traffic rates using large-scale real-world setup.}
\begin{tabularx}{\textwidth}{c|c|X|X|X|X|X|X|X|X|X}
\toprule
 \multicolumn{2}{c|}{\multirow{2}{*}{Method}} & \multicolumn{9}{c}{Traffic Rate (queries/s)}  \\ \cline{3-11}
 \multicolumn{2}{c|}{} & 731.534 & 1097.3  & 1463.067 & 1828.834  & 2194.601 & 2377.484  & 2560.368 & 2743.251  & 2926.135\\ \hline
 \multicolumn{2}{c|}{\multirow{2}{*}{LSQ}} & 0.014\newline$\pm$0.001  & 0.015\newline$\pm$0.000  & 0.015\newline$\pm$0.000 & 0.018\newline$\pm$0.003  & 0.217\newline$\pm$0.305  & 0.856\newline$\pm$0.554 & 11.066\newline$\pm$3.095  & 16.874\newline$\pm$0.391 & 17.155\newline$\pm$0.217\\ \hline
 \multicolumn{2}{c|}{\multirow{2}{*}{SED}}  & 0.014 \newline$\pm$0.000  & 0.015\newline$\pm$0.000  & 0.016\newline$\pm$0.001 & 0.018\newline$\pm$0.001  & 0.071\newline$\pm$0.066  & 1.252\newline$\pm$1.489 &  11.272\newline$\pm$3.975 & 21.941\newline$\pm$5.970 & 20.708\newline$\pm$5.423 \\ \hline
 \multirow{4}{*}{\makecell{\textbf{Distr-LB}\\(this paper)}}  &  \multirow{2}{*}{VBF} & 0.014\newline$\pm$0.000  & 0.015\newline$\pm$0.000 & 0.016\newline$\pm$0.001  & \textbf{0.017 \newline$\pm$0.000}  & \textbf{0.041\newline$\pm$0.025}  & \textbf{0.338 \newline$\pm$0.364} & 6.670\newline$\pm$1.152 & 9.743\newline$\pm$0.863 & 15.506\newline$\pm$0.056 \\ \cline{2-11}
 & \multirow{2}{*}{VBF+$\log$VBF}  &  0.014\newline$\pm$0.000 & 0.015\newline$\pm$0.001  & 0.016 \newline$\pm$0.000 & 0.018\newline$\pm$0.002  & 0.072\newline$\pm$0.087  & 0.465\newline$\pm$0.403  &  \textbf{3.970\newline$\pm$0.545} & \textbf{8.782\newline$\pm$0.187} & \textbf{15.095\newline$\pm$0.497} \\ 
\bottomrule
\end{tabularx}
\vskip -.1in
\label{tab:99_qos_static}
\end{table}

\textbf{Large-Scale Real-World Testbed:} To evaluate the performance of Distr-LB in large-scale DC networks in real world, we scale up the real-world testbed to have $6$ LB agents and $20$ servers and apply heavier network traffic ($>2000$ queries/s) to evaluate the performance of the LB algorithms that achieved the best performance in moderate scale setups, in comparison with in-production LB algorithms.
The test results after $200$ episodes of training are shown in Table~\ref{tab:compare_large_scale}, where Distr-LB achieves the best performance in all cases.
QMix also outperforms in-production LB algorithms. But as a CTDE algorithm, similar to the Centr-LB, it requires agents to communicate their trajectories, which -- after $200$ episodes of training -- become $221$MiB communication overhead at the end of each episode (episodic training), whereas $95\%$-percentile per-destination-rack flow rate is less than $1$MiB/s~\cite{facebook-dc-traffic}.

\textbf{Scaling Experiments:} Using the same large-scale real-world testbed with $6$ LB agents and $20$ servers, we conduct scaling experiments by applying network traces with different traffic rates, comparing $4$ LB methods with the best performances.
The $99$-th percentile QoS for both Wiki and static pages are shown in Table~\ref{tab:99_qos_wiki},~\ref{tab:99_qos_static}.
As listed in Table~\ref{tab:99_qos_wiki} and~\ref{tab:99_qos_static}, under low traffic rates, when servers are all under utilised, the advantage of our proposed Distr-LB is not obvious because all resources are over-provisioned. 
With the increase of traffic rates (till servers are $100\%$ saturated), our methods outperforms the best classical LB methods.
More in-depth discussion and analaysis over the average job completion time for both types of pages in these scaling experiments are shown in Table~\ref{tab:avg_wiki} and~\ref{tab:avg_static} in App.~\ref{app:results-ablation-comm}).

More details regarding the real-world DC testbed implementation is in App.~\ref{app:implement-testbed}, training details are in App.~\ref{app:hyperparameter}, complete evaluation results (both moderate-scale and large-scale) are in App.~\ref{app:results} and ablation studies -- \eg communication overhead of CTDE and centralised RL in real-world system, robustness of MARL algorithms in dynamic DC network environments -- can be found in App.~\ref{app:results-ablation}.

\section{Conclusion and Future Work}
\label{sec:conclusion}
This paper proposes a distributed MARL approach for multi-agent load balancing problem, based on Markov potential game formulation. The proposed variance-based fairness for individual LB agent is critical for this formulation. Through this setting, the redundant communication overhead among LB agents is removed, thus improving the overall training and deployment efficiency in real-world systems, with the local observations only. Under such formulation, the effectiveness of our proposed distributed LB algorithm together with the proposed fairness are both theoretically justified and experimentally verified. It demonstrates a performance gain over another commonly applied fairness as well as centralised training methods like QMIX or centralised RL agent, in both simulation and real-world tests with different scales.

\bibliographystyle{unsrt}
\bibliography{reference}

\newpage

\appendix
\section*{Appendix}

\section{A Stochastic Markov Model of a $2$-Server Load Balancing Problem}
\label{app:model-basic}

The simulation results of Fig.~\ref{fig:motivation-inaccurate} is based on a basic load balancing setup of $2$ servers with different processing capacities $\frac{v_1}{v_2} = 2$ (\ie server $1$ is 2x faster than server $2$).
Each server has a queue of size $Q$, such that $0\leq l_1, l_2 \leq Q$.
Traffic arrivals and departures are modeled as Poisson processes with rates $\lambda$ (observed traffic), $\gamma$ (unobserved traffic), and $v_1$, $v_2$.
With sufficiently short timeslots, it can be assumed that only one arrival or departure (at most) happen at a given timeslot (\ie $\sum_{i=1}^2 (\lambda_i + \gamma_i + v_i) \leq 1$); the system is then Markovian with the state $(l_1, l_2)$, departure rates $(\mu_1, \mu_2)$, and arrival rates $(\lambda_1, \lambda_2, \gamma_1, \gamma_2)$.
For simplicity and stability, the system works at {\em nominal} capacity (\ie $\lambda + \gamma = v$).
With $q_i(n)_{l_i}$ denoting the probability (or probability density function), of server $q_i$ to have a queue length of $l_i$ at time-step $n$, the transition of server occupations between two time-steps can be described as, for $0 < l_i < Q$ (corner cases are treated accordingly):
{\small
\begin{align*}
	q_i(n)_{l_i} - q_i(n-1)_{l_i} = ( \lambda_i + \gamma_i ) \cdot q_i(n-1)_{l_i-1} + v_i \cdot s_i(n-1)_{l_i+1} - (\lambda_i+\gamma_i+v_i) \cdot q_i(n-1)_{l_i}.
\end{align*}
}%

The QoS performance of each load balancing algorithm in Fig.~\ref{fig:motivation-inaccurate} is measured as the weighted service duration of a connection ($\sum_{i\in\{1, 2\}} \frac{l_i}{l_1+l_2} \frac{l_i}{\mu_i}$), under different configurations.
When the LB has accurate observations and configurations (observing $100\%$ traffic -- \ie $\gamma = 0$ -- and assigning server weights based on actual processing speeds $\frac{w_1}{w_2}=\frac{v_1}{v_2}=2$), WCMP and SED have the best performance.
When the LB observes only partial network traffic ($50\%-Q$ and $33\%-Q$ corresponds to $\gamma = \lambda$, $\gamma = 2*\lambda$, respectively) and the rest of the network traffic is uniformly split between the two servers ($\gamma_1 = \gamma_2$), LSQ and SED outperform WCMP, which is agnostic to instant server occupancy.
However, partial traffic observation also degrades the performance of LSQ and SED.
When LBs have inaccurate server weights ($\sim W$ \ie in case of mis-configuration, $\frac{w_1}{w_2}=\frac{1}{2}$, while $\frac{\mu_1}{\mu_2}=2$), WCMP and SED exhibit degraded performance even when the LB agent sees all the traffic ($\gamma = 0$).
Taking both server queue lengths and processing speeds into account, SED makes more informed load balancing decisions, yet its performance risks being degraded by both partial observations on server queue lengths and inaccurate server weights.

\section{Analysis of Distribution Fairness}

\subsection{Analysis of VBF}
\label{sec:app_vbf}
\begin{lemma} The VBF for load balancing system satisfies the following property:
\begin{align}
    F_i^{\pi_i, -\pi_i}(\boldsymbol{l}_i)-F_i^{\tilde{\pi}_i, -\pi_i}(\tilde{\boldsymbol{l}}_i) = F^{\pi_i, -\pi_i}(\boldsymbol{l})-F^{\tilde{\pi}_i, -{\pi}_i}(\tilde{\boldsymbol{l}})
\end{align}
\end{lemma}

\begin{proof}
From the definition of the variance-based fairness (as Def.~\ref{def:vbf}) we have the following for $\forall i\in[M], j\in[N]$,
\begin{align}
    F^{\pi_i, -\pi_i}(\boldsymbol{l}) &= -\frac{1}{N}\sum_{j=1}^{N}(l_j-\overline{\boldsymbol{l}})^2\\
    F_i^{\pi_i, -\pi_i}(\boldsymbol{l}_i)&= -\frac{1}{N}\sum_{j=1}^{N}(l_{ij}-\overline{l_i})^2 \quad (\overline{l_i}=\frac{1}{N}\sum_{j=1}^{N}l_{ij})
\end{align}
By indexing the agent $i$ as the one to change its strategy and slightly abusing notation, denote $l_j = l_{ij} + l_{-ij}$, where $l_{-ij} = \sum_{k\neq i}l_{kj}$.
\begin{align}
     F^{\pi_i, -\pi_i}(\boldsymbol{l})&=-\frac{1}{N}\sum_{j=1}^{N}(l_{ij}+l_{-ij}-\overline{(l_{i}+l_{-i})})^2 \quad (\text{where }\overline{(l_{i}+l_{-i})}=\frac{1}{N}\sum_j(l_{ij}+l_{-ij}))\\
     &=-\frac{1}{N}\sum_{j=1}^{N}[l_{ij}+l_{-ij}-(\overline{l}_{i}+\overline{l}_{-i})]^2\\
    &=-\frac{1}{N}\sum_{j=1}^{N}[(l_{ij}-\overline{l}_{i})^2+(l_{-ij}-\overline{l}_{-i})^2-2(l_{ij}-\overline{l}_{i})(l_{-ij}-\overline{l}_{-i})]\\
    &=-\frac{1}{N}\sum_{j=1}^{N}(l_{ij}-\overline{l_{i}})^2-\frac{1}{N}\sum_{j=1}^{N}[(l_{-ij}-\overline{l}_{-i})^2-\frac{2}{N}\sum_{j=1}^{N}(l_{ij}-\overline{l}_{i})(l_{-ij}-\overline{l}_{-i})]\\
    &=F_i^{\pi_i, -\pi_i}(\boldsymbol{l}_i)-\frac{1}{N}\sum_{j=1}^{N}(l_{-ij}-\overline{l}_{-i})^2 \quad (\sum_{j=1}^{N}(l_{ij}-\overline{l}_i)=0)
\end{align}
where the second term is a common term not depend on the changing policy $\pi_i$. Therefore, the second term will be cancelled in $ F^{\pi_i, -\pi_i}(\boldsymbol{l})-F^{\tilde{\pi}_i, -\pi_i}(\tilde{\boldsymbol{l}})=F_i^{\pi_i, -\pi_i}(\boldsymbol{l}_i)-F_i^{\tilde{\pi}_i, -\pi_i}(\tilde{\boldsymbol{l}}_i)$, thus finishes the proof.
\end{proof}

\begin{proposition}
\label{prop:vbf-proof}
Maximising the VBF is sufficient for minimising the makespan, subjective to the load balancing problem constraints (Eq.~\eqref{eq:lb_cons1} and \eqref{eq:lb_cons2}):
\begin{align}
    \max F(\boldsymbol{l}) \Rightarrow  \min \max_j(l_j)
\end{align}
this also holds for per-LB VBF as $\max F_i(\boldsymbol{l}_i) \Rightarrow  \min \max_j(\boldsymbol{l}_i)$.
\end{proposition}

\begin{proof}
Given the stability constraint in Eq.~\eqref{eq:lb_cons1} $\sum_{i=1}^{M}w_i(t)\leq\sum_{j=1}^{N}v_j$, we denote the total amount of workload in the system $C=\sum_{j=1}^{N}l_j$, and $l_k=\max_{j\in[N]}l_j$.
Based on the constraint in Eq.~\eqref{eq:lb_cons2}, we have $C \geq 0$, $l_j(t) \geq 0$.
\begin{align}
\max F(\boldsymbol{l}) &\Leftrightarrow \min -F(\boldsymbol{l})\\
-F(\boldsymbol{l}) &= \frac{1}{N}\sum_{j=1}^{N}((l_j)-\overline{\boldsymbol{l}})^2\\
&=\frac{1}{N}\sum_{j=1}^{N}(l_j - \frac{C}{N})^2\\
&=\frac{1}{N}\sum_{j=1}^{N}l^2_j - \frac{2C}{N^2}\sum_{j=1}^{N}l_j+\frac{C^2}{N^2}\\
&=\frac{1}{N}\sum_{j=1}^{N}l^2_j-\frac{C^2}{N^2}\\
&\le [(\max_j l_j)^2-\frac{C^2}{N^2}] \quad (\text{by means inequality})
\end{align}
with the equivalence achieved when $l_j=l_k, \forall j\neq k, j \in [N]$ holds.
Therefore,
\begin{align}
    \max F(\boldsymbol{l}) &\Rightarrow \min (l_k)^2-\frac{C^2}{N^2}  \\
    & \Leftrightarrow \min l_k\\
    &\Leftrightarrow \min \max_{j\in[n]} l_j
\end{align}
and the condition is sufficient but not necessary because $\min (l_k)^2-\frac{C^2}{N^2}$ is essentially minimizing the upper bound of $-F(\boldsymbol{l})$.
\end{proof}

\subsection{Analysis of PBF}
\label{sec:app_pbf}
\begin{proposition}
Maximising the product-based fairness is sufficient for minimising the makespan, subjective to the load balancing problem constraints (Eq.~\eqref{eq:lb_cons1} and \eqref{eq:lb_cons2}):
\begin{align}
    \max F(\boldsymbol{l}) \Rightarrow  \min \max(\boldsymbol{l})
\end{align}
\end{proposition}

\begin{proof}
For a vector of workloads $\boldsymbol{l}=[l_1, \dots, l_N]$ on each server $j\in[N]$, by the definition of fairness, 
\begin{align}
    \max F(\boldsymbol{l}) &= \max \frac{\prod_{j\in[N]} l_j}{\max_{k\prime\in[N]} l_{k^\prime}}
\end{align}
WLOG, let $l_k = \max_{k^\prime\in[N]} l_{k^\prime}$, then,
\begin{align}
     \max F(\boldsymbol{l}) = \max \prod_{j\in[N], j \neq k} l_j
\end{align}
Similar to the proof of Proposition~\ref{prop:vbf-proof}, given the stability constraint in Eq.~\eqref{eq:lb_cons1} $\sum_{i=1}^{M}w_i(t)\leq\sum_{j=1}^{N}v_j$, we denote the total amount of workload in the system $C=\sum_{j=1}^{N}l_j$.
Based on the constraint in Eq.~\eqref{eq:lb_cons2}, we have $C \geq 0$, $l_j(t) \geq 0$.
By means inequality,
\begin{align}
    \left(\prod_{j\in[N], j \neq k}l_j\right)^{\frac{1}{N-1}} \leq \frac{\sum_{j\in[N], j \neq k}l_j}{N-1} = \frac{C-l_k}{N-1}.
\end{align}
with the equivalence achieved when $l_i=l_j, \forall i,j\neq k, i,j\in[N]$ holds.
Therefore,
\begin{align}
    \max F(\boldsymbol{l}) &\Rightarrow \max \frac{C-l_k}{N-1} \\
    & \Leftrightarrow \min l_k\\
    &\Leftrightarrow \min \max_{j\in[N]} l_j
\end{align}
The inverse may not hold since $\max \frac{C-l_k}{N-1}$ does not indicates $\max F(\boldsymbol{l})$, so maximising the linear product-based fairness is sufficient but not necessary for minimising the makespan. This finishes the proof.
\end{proof}

\subsection{VBF for MPG}
\label{sec:app_vbf_mpg}
\begin{theorem}
Multi-agent load balancing is MPG with the VBF $F_i(\boldsymbol{l}_i)$ as the reward $r_i$ for each LB agent $i\in[M]$, then suppose for $\forall s\in \mathcal{S}$ at step $h\in[H]$, the potential function is time-cumulative total fairness: $\phi^{\pi_i, -\pi_i}(s)=\sum_{t =h}^H F^{\pi_i, -\pi_i}(\boldsymbol{l}(t))$.
\end{theorem}
\begin{proof}
\begin{align}
    V_i^{\pi_i, \pi_{-i}}(s) - V_i^{\tilde{\pi}_i, \pi_{-i}}(s)&=\mathbb{E}_{\pi_i, \pi_{-i}}\bigg[\sum_{t =
        h}^H r_{i, t}(s_{t}, \boldsymbol{a}_{t}) \bigg| s_h = s\bigg] - \mathbb{E}_{\tilde{\pi}_i, \pi_{-i}}\bigg[\sum_{t =
        h}^H r_{i, t}(s_{t}, \tilde{a}_{i, t}, a_{-i, t}) \bigg| s_h = s\bigg]\\
        &=\mathbb{E}_{{\pi}_i, \pi_{-i}}\bigg[\sum_{t =
        h}^H F_i(\boldsymbol{l}_i(t))\bigg] - \mathbb{E}_{\tilde{\pi}_i, \pi_{-i}}\bigg[\sum_{t =
        h}^H F_i(\tilde{\boldsymbol{l}}_i(t))\bigg]\\
        &=\sum_{t =
        h}^H \bigg(F^{\pi_i, -\pi_i}(\boldsymbol{l})-F^{\tilde{\pi}_i, -\pi_i}(\tilde{\boldsymbol{l}})\bigg) \quad (\text{Lemma \ref{lem:vbf}})\\
        &=\phi^{\pi_i, -\pi_i}(s)-\phi^{\tilde{\pi}_i, -\pi_i}(s)
\end{align}
Notice that $s$ is the ground truth state of the environment, therefore involving the expected time $\boldsymbol{l}$ to finish remaining jobs.
\end{proof}

\begin{lemma} NE for MPG is $\epsilon$-approximate NE for $\epsilon$-approximate MPG.~\cite{ali2019reinforcement}

\begin{proof}
We know NE $(\pi^*_i, \pi^*_{-i})$ for MPG,
\begin{align}
    V_i^{\pi^*_i, \pi^*_{-i}}(s) - V_i^{\tilde{\pi}_i, \pi^*_{-i}}(s)=\phi^{\pi^*_i, {\pi}^*_{-i}}(s)-\phi^{\tilde{\pi}_i, \pi^*_{-i}}(s) \ge 0
\end{align}
the policies can be $\epsilon$-approximate NE for another game with a different value function $\widehat{V}$ but the same potential function,
\begin{align}
    \widehat{V}_i^{\pi^*_i, \pi^*_{-i}}(s) - \widehat{V}_i^{\tilde{\pi}_i, \pi^*_{-i}}(s) \ge \epsilon, \forall i\in[N], \tilde{\pi}_i\in\Pi_i, s\in\mathcal{S}
\end{align}
thus,
\begin{align}
   \bigg|\bigg( \widehat{V}_i^{\pi^*_i, \pi^*_{-i}}(s) - \widehat{V}_i^{\tilde{\pi}_i, \pi^*_{-i}}(s) \bigg)-\bigg( \phi^{\pi^*, \pi^*_{-i}}(s)-\phi^{\tilde{\pi}, \pi^*_{-i}}(s)\bigg)\bigg| \le \epsilon
\end{align}
which satisfies the definition of $\epsilon$-approximate MPG.
\end{proof}
\end{lemma}

\section{Implementation}
\label{app:implement}

\subsection{Simulator}
\label{app:implement-simulator}

\begin{figure}[tbp]
	\centering
	\begin{subfigure}{0.48\columnwidth}
		\centering
		\includegraphics[width=0.9\columnwidth]{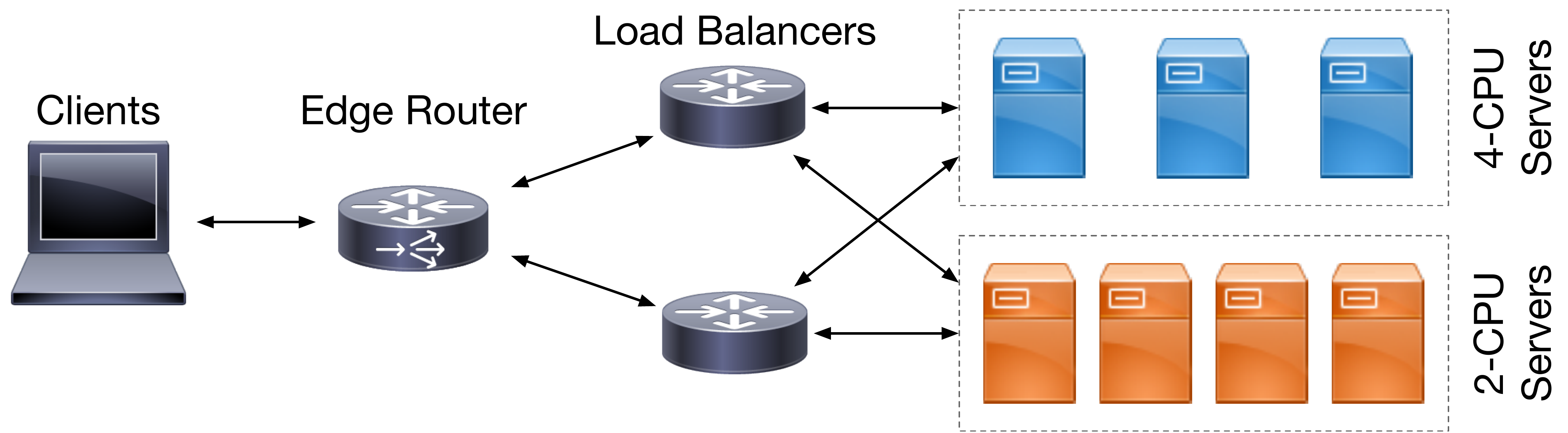}
		\caption{An example of network topology with two groups of $7$ servers.}
		\label{fig:app-implement-topo}
	\end{subfigure}
	\hspace{.05in}
	\begin{subfigure}{0.48\columnwidth}
		\centering
		\includegraphics[width=0.9\columnwidth]{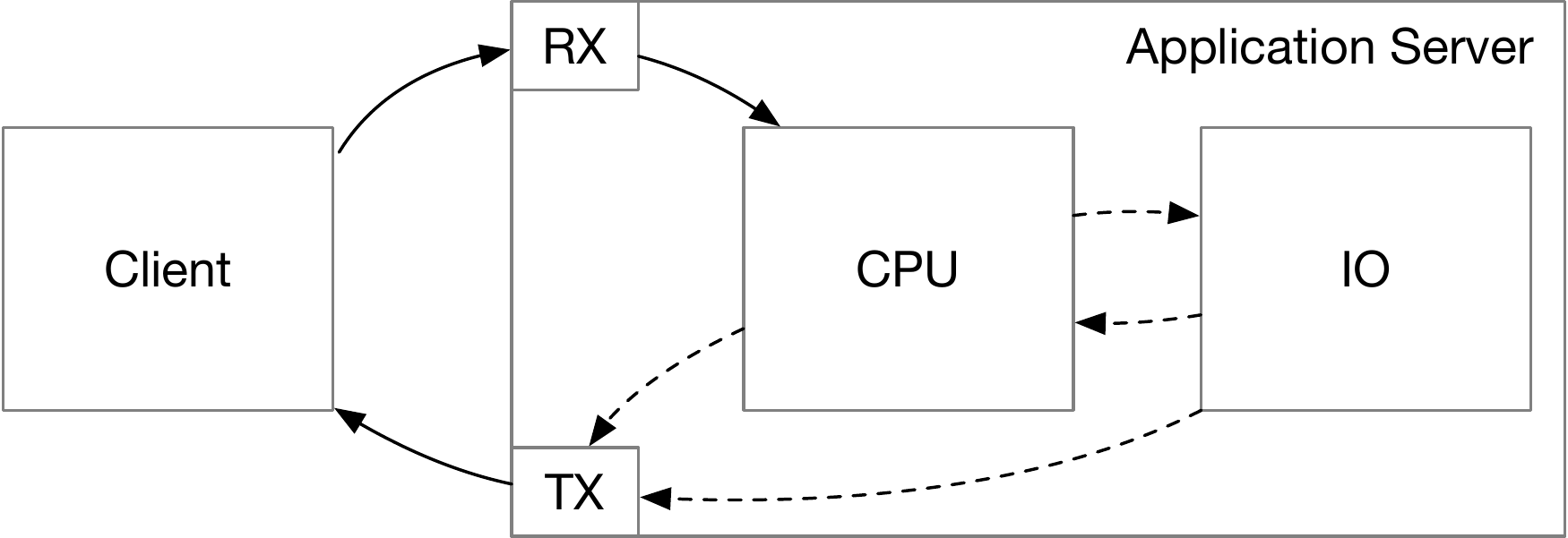}
		\caption{Illustration of the processing states of connection requests. Solid and dashed arrows represent deterministic and non-deterministic procedures respectively.}
		\label{fig:app-implement-simulator-process-model}
	\end{subfigure}
	\caption{Simulator implementation details.}
	\label{fig:app-implement-simulator}
\end{figure}

In order to compare the proposed RLB algorithms to the theoretically optimal solution which has perfect observation over the system -- which is not achievable in real-world system, we implement an event-driven simulator to simulate the discrete process of network flow arrival and departure in a networked system.
The simulator implements the network topology as in Fig.~\ref{fig:app-implement-topo}, where each LB is connected to all servers.

Real-world network applications can be CPU-bound or IO-bound~\cite{hadoop2011apache, spark2018apache}.
The simulator allows configuring applications that require multi-stage processes switching between CPU/IO queues (Fig.~\ref{fig:app-implement-simulator-process-model}). For instance, a connection request for a $2$-stage application is first processed in the CPU queue, then in the IO queue, before being sent back to the client.

Two different processing models are used for CPU and IO queues, respectively.
A FIFO model is defined for CPU queues, and connections that arrive when no CPU is available will be blocked in a backlog queue until there is an available CPU.
Realistic network applications feature blocked processor sharing model~\cite{hadoop2011apache, spark2018apache}, in which the instantaneous processing speed for each task $\hat{v}_j(t)$ at time $t$ on the $j$-th server is:
\begin{equation}
	\hat{v}_j(t)=
	\begin{cases}
	1       & \quad |w_{j}(t)| \leq p_{j},\\
	\frac{p_{j}}{\min \left(\hat{p}_{j}, |w_{j}(t)|\right)}  & \quad |w_{j}(t)| > {p}_{j},
	\end{cases}
\end{equation}
where $|w_{j}(t)|$ denotes the number of on-going tasks, and $p_j$ denotes the number of processors on the $j$-th server.
At any given moment, the maximum number of tasks that can be processed is $\hat{p}_j$.
Tasks that arrive when $|w_{j}(t)| \geq \hat{p}_j$ will be blocked in a wait queue (similar to backlog in \eg Apache) and will not be processed until there is an available slot in the CPU processing queue.
However, this does not happen under the constraints in Eq.~\eqref{eq:lb_cons1} as the task arrival rates are always slower than task departure rates (processing speed).
The server processing speed therefore is $v_j(t)=\hat{v}_j(t)|w_{j}(t)|$.
IO is simulated as a simple processor sharing model, in which the instantaneous processing speed is the inverse of the number of connections in the IO queue.
The backlog queue length of each server is configured as $64$.
Connections that arrive when the backlog queues are full will be rejected, with $40$s timeout.
Communication latency between $2$ nodes is uniformly distributed between $0.1$ms and $1$ms.

\subsection{Real-World DC Testbed}
\label{app:implement-testbed}

\subsubsection{System Platform}
\label{app:testbed-system}

Application servers are virtualised on $4$ UCS B200 M4 servers, each with one Intel Xeon E5-2690 v3 processor ($12$ physical cores and $48$ logical cores), interconnected by UCS $6332$ $16$UP fabric.
Operating systems are \texttt{Ubuntu 18.04.3 LTS} (\texttt{GNU/Linux 4.15.0-128-generic x86\_64}).
Compilers are \texttt{gcc} version \texttt{7.5.0} (\texttt{Ubuntu 7.5.0-3ubuntu1~18.04}).
Applications employed in this paper are the following: \texttt{Apache} \texttt{2.4.29}, \texttt{VPP v20.05}, \texttt{MySQL 5.7.25-0ubuntu0.18.04.2}, and \texttt{MediaWiki v1.30}.
The VMs are deployed on the same layer-$2$ link, with statically configured routing tables. 
Apache HTTP servers share the same VIP address on one end of GRE tunnels with the load balancer on the other end.

\subsubsection{Apache HTTP Servers}
\label{app:testbed-apache}

The Apache servers use \texttt{mpm\_prefork} module to boost performance.
Each server has maximum $32$ worker threads and TCP backlog is set to $128$.
In the Linux kernel, the \texttt{tcp\_abort\_on\_overflow} parameter is enabled, so that a TCP RST will be triggered when the queue capacity of TCP connection backlog is exceeded, instead of silently dropping the packet and waiting for a SYN retransmit.
With this configuration, the FCT measures application response delays rather than potential TCP SYN retransmit delays.
Two metrics are gathered as ground truth server load state on the servers: CPU utilization and instant number of Apache busy threads.
CPU utilization is calculated as the ratio of non-idle cpu time to total cpu time measured from the file \texttt{/proc/stat} and the number of Apache busy threads is assessed via Apache's \textit{scoreboard} shared memory.

\subsubsection{$24$-Hour Wikipedia Replay Trace}
\label{app:testbed-wiki}

To create Wikipedia server replicas, an instance of MediaWiki\footnote{https://www.mediawiki.org/wiki/Download} of version $1.30$, a MySQL server and the \texttt{memcached} cache daemon are installed on each of the application server instance. 
\textit{WikiLoader} tool~\cite{wikiloader} and a copy of the English version of Wikipedia database~\cite{wiki_traces}, are used to populate MySQL databases. 
The 24-hour trace is obtained from the authors of~\cite{wiki_traces} and for privacy reasons, the trace does not contain any information that exposes user identities.

\subsubsection{Feature Collection and Policy Update in the Data Plane}
\label{app:testbed-feature-collection}

\begin{figure}[t]
	\centering
	\centerline{\includegraphics[width=.6\columnwidth]{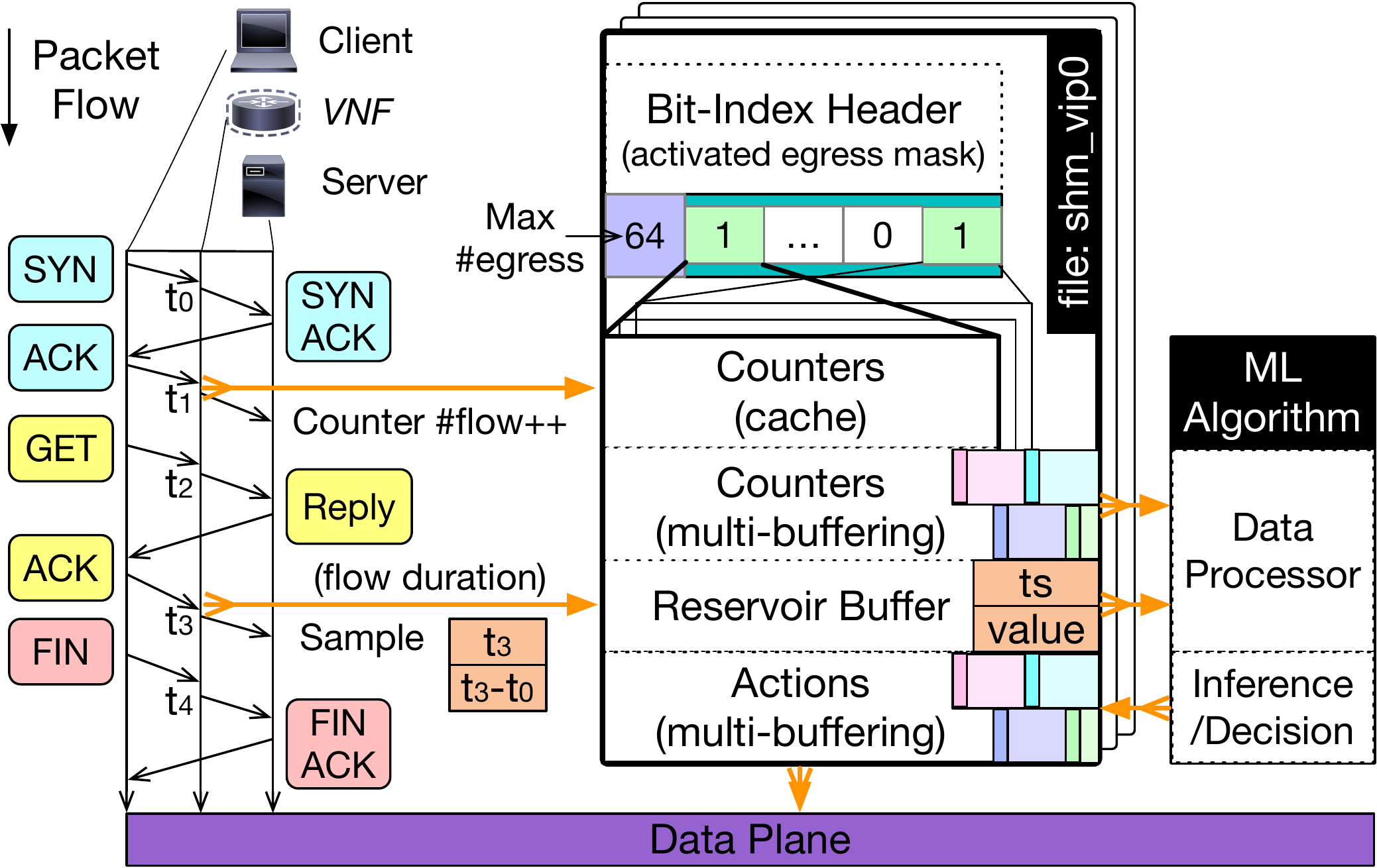}}
	\caption{Feature collection mechanism: \texttt{shm} layout and data flow pipeline.}
	\label{fig:app-design-sim}
\end{figure}

\begin{algorithm}[H]
	\footnotesize
	\caption{Reservoir sampling with no rejection}
		\label{alg:feature-reservoir}
		\begin{algorithmic}[1]
		\STATE $k \gets$ reservoir buffer size
		\STATE $buf \gets [(0, 0), \dots, (0, 0)]$\Comment{Size of $k$}
		\FOR {each observed sample $v$ arriving at $t$}
			\STATE $randomId \gets rand()$
			\STATE $idx \gets randomId\%N$ \Comment{randomly select one index}
			\STATE $buf[idx] \gets (t, v)$ \Comment{register sample in buffer}
		\ENDFOR
		\end{algorithmic}
\end{algorithm}

In order to apply RL in an asynchronous close-loop load balancing framework with high scalability and low latency, communication between the load balancer data plane and the ML application is implemented via POSIX shared memory (\texttt{shm}).
This mechanism allows features to be extracted from the data plane and conveyed to the RL agent -- with absolutely zero control message or communication overhead, and allows data-driven decisions generated by the RL agent to be updated asynchronously on the load balancer.

The pipeline of the data flow over the lifetime of a TCP connection is depicted in Fig.~\ref{fig:app-design-sim}.
By statefully tracking flow states, on receipt of different networking packets, we inspect packet header and collect networking features as counters or samples.
Quantitative features (task duration and task completion time) are collected as samples, using reservoir sampling (Algorithm~\ref{alg:feature-reservoir}).
Since networking environments are dynamic, it is important to capture not only the features, but also the sequential information of the system.
Reservoir sampling gathers a representative group of samples in fix-sized buffer from a stream, with $\mathcal{O}(1)$ insertion time.
It captures both the sampling timestamps and exponentially-distributed numbers of samples over a time window, which help conduct sequential pattern analysis\footnote{Based on the characteristics of different system dynamics, \eg long-term distribution shifts or short-term oscillations, tuning the reservoir sampling mechanism enables to collect different statistical representations of the states.}.
For a Poisson stream of events with rate $\lambda$, the expectation of the amount of samples that are preserved in buffer after $n$ steps is $E = \lambda \left(\frac{k - 1}{k}\right)^{\lambda n}$, where $k$ is the size of reservoir buffer.
On the other hand, counters are collected atomically and made available to the data processing agent using multi-buffering.

Cloud services have different characteristics and they are identified by virtual IPs (VIPs), which correspond to clusters of provisioned resources -- \eg servers, identified by a unique direct IP (DIP).
In production, cloud DCs are subject to high traffic rates and their environments and topologies change dynamically.
This requires to organise collected features in a generic yet scalable format, and make features available for ML algorithms without disrupting the data plane.
We organise observations of each VIPs in independent POSIX shared memory (\texttt{shm}) files, to provide scalable and dynamic service management.
In each \texttt{shm} file, collected features are further partitioned by egress equipments so that spatial information can be distinguished, including counters and reservoir samples.
Fig.~\ref{fig:app-design-sim} exemplifies the \texttt{shm} layout and data flow.

The bit-index binary header helps efficiently identify active application servers.
Each server has its own independent memory space, storing counters, reservoir samples, and data plane policies (actions) if necessary.
As depicted in Fig.~\ref{fig:app-design-sim}, on receipt of the first \texttt{ACK} from the client to a specific server $i$, VNF increments the number of flows in the counters cache of node $i$ with $\mathcal{O}(1)$ complexity.
With the same level of complexity, quatitative features (\eg flow duration $t_3 - t_0$ gathered at $t_3$ in Fig.~\ref{fig:app-design-sim}) can be stored in the reservoir buffer of node $i$ using Algorithm~\ref{alg:feature-reservoir}.
Gathered features (counters and samples) are made available in an organised layout and they can be quickly accessed by ML algorithms running in a different process.
With the bit-index header, locating features for a given server requires $\mathcal{O}(1)$ computational complexity and $\mathcal{O}(k)$ memory complexity, where $k$ is the reservoir buffer size.
Obtained features for all active servers can then be aggregated and processed to make further inferences or data-driven decisions, which can be written back to the memory space of each server ($\mathcal{O}(1)$ computational complexity).

While quantitative features are collected using reservoir sampling, counters are incremented by the data plane in the cache, and periodically drawn from cache using $m$-level multi-buffering with incremental sequence ID.
When copying data between cache and buffer, the sequence ID is set to $0$ to avoid I/O conflicts.
Pulling the counters from cache to multi-buffering requires $\mathcal{O}(1)$ computational complexity and maximal $\mathcal{O}(N)$ memory complexity.
ML algorithms can pull the latest observations from the multi-buffering with no disruption in the data plane, with $\mathcal{O}(m)$ computational complexity to find the buffer with the highest sequence ID.
Similarly, new network policies and data-driven decisions (\eg forwarding rules) can be updated to the data plane via action multi-buffering with $\mathcal{O}(m)$ computational complexity.

\begin{table}[t]
	\footnotesize
	\centering
	\begin{tabular}{|ll|c|c|}
		\hline
		\multicolumn{2}{|l|}{Operation / Complexity}                                                                                        & Computation                       & Memory               \\ \hline
		\multicolumn{2}{|l|}{Add / Remove VIP}                                                                                              & $\mathcal{O}(1)$                  & $\mathcal{O}(kN+mN)$ \\ \hline
		\multicolumn{2}{|l|}{Add server}                                                                                               & $\mathcal{O}(1)$                  & $\mathcal{O}(k+m)$   \\ \hline
		\multicolumn{2}{|l|}{Remove server}                                                                                            & $\mathcal{O}(1)$                  & $\mathcal{O}(1)$     \\ \hline
		\multicolumn{2}{|l|}{\begin{tabular}[c]{@{}l@{}}Register reservoir sample\\ Update counter (cache)\end{tabular}}                    & $\mathcal{O}(1)$                  & $\mathcal{O}(1)$     \\ \hline
		\multicolumn{2}{|l|}{\begin{tabular}[c]{@{}l@{}}Update counters / actions\\ (multi-buffering)\end{tabular}}                         & $\mathcal{O}(1)$                  & $\mathcal{O}(N)$     \\ \hline
		\multicolumn{1}{|l|}{\multirow{2}{*}{\begin{tabular}[c]{@{}l@{}}Get the latest \\observation\end{tabular}}}                                                       & 1 node    & \multirow{2}{*}{$\mathcal{O}(m)$} & $\mathcal{O}(k+m)$   \\ \cline{2-2} \cline{4-4} 
		\multicolumn{1}{|l|}{}                                                                                                  & All nodes &                                   & $\mathcal{O}(kN+mN)$ \\ \hline
		\multicolumn{1}{|l|}{\multirow{2}{*}{\begin{tabular}[c]{@{}l@{}}Update action in\\ the data plane\end{tabular}}} & 1 node    & \multirow{2}{*}{$\mathcal{O}(m)$} & $\mathcal{O}(1)$     \\ \cline{2-2} \cline{4-4} 
		\multicolumn{1}{|l|}{}                                                                                                  & All nodes &                                   & $\mathcal{O}(N)$     \\ \hline
	\end{tabular}
    \caption{Computation and memory complexity of different operations, where $k$ is the size of reservoir buffer, $N$ is the number of servers, and $m$ is the level of multi-buffering.}
	\label{tab:design-complexity}
\end{table}

To summarise, both computation and memory space complexity is presented in Table~\ref{tab:design-complexity}.
The whole dataflow is asynchronous and avoid stalling in the data exchange process in both the data plane and the control plane.

\subsubsection{Network Topology}
\label{app:testbed-topology}

\begin{figure}[t]
	\centering
	\centerline{\includegraphics[width=\columnwidth]{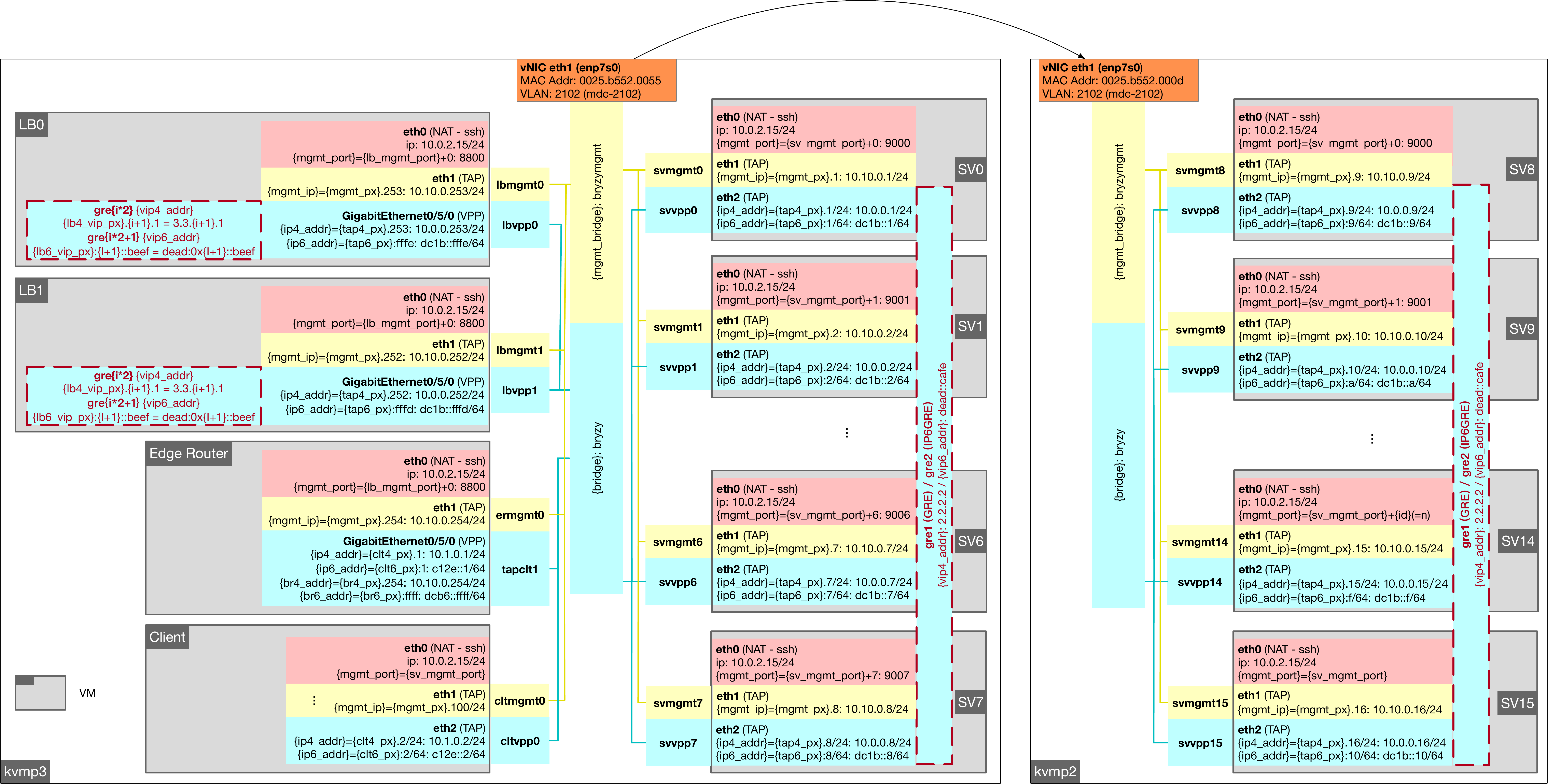}}
	\caption{Network topology of the real-world DC testbed.}
	\label{fig:app-design}
\end{figure}

For reproducibility, the network topology is depicted in Fig.~\ref{fig:app-design}. Two physical servers are connected via a VLAN. Each device is an instance of KVM, which is widely used for in-production vitualised Content Delivery Networks (CDNs).

\subsubsection{Realistic Testbed}

Modern data center may comprise thousands of servers and hundreds of LBs.
However, each independent service is exposed in a modular way at one or several virtual IP (VIP) addresses to receive requests, running over a cluster of servers.
Each server in the cluster can be identified by a unique direct IP (DIP).
Traffic and queries from the clients destined to a VIP are load balanced among the DIPs of the service.
The development of virtualization, where computers are emulated and/or sharing an isolated portion of the hardware by way of Virtual Machines (VMs), or run as isolated entities (containers) within the same operating system kernel, has accelerated the commoditization of compute resources.
Therefore, the gigantic in-production data center network are typically partitioned into small pods, where different services (VIPs) are hosted.
To justify the setups of our experiment satisfy the ``real-world'' requirement, we present a brief survey of real-world DC setup based on a set of state-of-the-art load balancing research papers, which are summarized below (Table~\ref{tab:dc-survey}).

\begin{table}[tbp]
	\caption{Survey on real-world testbed configurations.}
	\label{tab:dc-survey}
	\centering
	\resizebox{\textwidth}{!}{ 
	\begin{tabular}{lll}
	\toprule
	Related Work & Testbed Scale                                                                                & Note                                                                                                                                                                          \\ \midrule
	6LB~\cite{6lb}         & 2 LB + 28 servers (2-CPU each)                                                               & Our paper uses the same network trace as input traffic.                                                                                                                       \\ \cline{1-3}
	Ananta~\cite{ananta2013}       			& 14 LBs for $12$ VIPs                                                                         & \begin{tabular}[c]{@{}l@{}}The exact number of servers per VIP and the in-production\\  traffic is not documented in the paper.\end{tabular}                                  \\ \cline{1-3}
	Beamer~\cite{beamer}       & \begin{tabular}[c]{@{}l@{}}2 LB + 8 servers (small)\\ 4 LB + 10 servers (large)\end{tabular} & \begin{tabular}[c]{@{}l@{}}Large scale experiments are conducted with 700 active\\ HTTP connections max.\end{tabular}                                                         \\ \cline{1-3}
	Duet~\cite{duet}         & \begin{tabular}[c]{@{}l@{}}3 software LB + 3 hardware LB\\ + 34 servers\end{tabular}         & \begin{tabular}[c]{@{}l@{}}Synthetic traffic is applied so that the server cluster \\ behind VIP processes 60k (identical) packets per second.\end{tabular}                   \\ \cline{1-3}
	SilkRoad~\cite{silkroad2017}     & \begin{tabular}[c]{@{}l@{}}1 hardware LB or 3 software LB\\ per VIP\end{tabular}             & \begin{tabular}[c]{@{}l@{}}Real-world PoP traffic is applied, where one server cluster\\ behind VIP processes on average 309.84 active connections\\ per second.\end{tabular} \\ \cline{1-3}
	Cheetah~\cite{cheetah2020}      & 2 LB + 24 servers                                                                            & A Python generator creates 1500-2500 synthetic requests/s.                                                                                                                    \\ \bottomrule
	\end{tabular}
	}
\end{table}

Using $2$ physical servers ($48$ CPUs each), we have made our best effort to find a configuration that allows us to conduct experiments similar to real-world setups.
Based on the survey above, we believe that the experiments conducted in this paper have reasonable scale -- not only in terms of number of agents ($2/6$ LBs) and servers ($7/20$ servers), but also in terms of traffic rates -- more than $2$k queries per second per VIP and more than $1150.76$ concurrent connections in large scale experiments —- and are representative of real-world circumstances.

\subsection{Benchmark Load Balancing Methods}
\label{sec:simulator-methods}

To compare load balancing performance, $4$ state-of-the-art workload distribution algorithms are implemented.
Equal-cost multi-path (ECMP) randomly assigns servers to tasks with a server assignment function $\mathbb{P}(j) = \frac{1}{n}$, where $\mathbb{P}(j)$ denotes the probability of assigning the $j$-th server~\cite{faild2018}.
Weighted-cost multi-path (WCMP) assigns servers based on their weights derived, and has an assignment function as $\mathbb{P}(j) = \frac{v_j}{\sum v_j}$~\cite{maglev}.
Local shortest queue (LSQ) assigns the server with the shortest queue, \ie $\arg \min_{j\in[n]} |w^j(t)|$~\cite{twf2020}.
Shortest expected delay (SED) assigns the server the shortest queue normalized by the number of processors, \ie $\arg \min_{j\in[n]} \frac{|w^j(t)|+1}{v_j}$~\cite{lvs}, and is expected to have the best performance among conventional heuristics.
In the simulator, an \textit{Oracle} LB algorithm is implemented, which distributes connections to the server which is expected to finish all its job with the lowest delay (including the new connection).
The Oracle LB is aware of the remaining time of each connection, which is otherwise not observable for network LBs in real-world setups.
When receiving a new connection, the Oracle LB algorithm calculates the remaining time to process on each server (assuming the newly received connection is assigned on the server as well) and assigns the server with the lowest remaining time to process to the new-coming connection, to make sure that the makespan is always minimised with the global observation, which is not possible to be achieved in real-world system.
The load balancing decisions for the Oracle algorithm are also made immediately for the Oracle LB algorithm.

\section{Hyperparameters}
\label{app:hyperparameter}

\begin{table}[tbp]
\caption{Hyperparameters in MARL-based LB.}
\label{tab:rl_params}
\centering
\resizebox{\textwidth}{!}{ 
\begin{tabular}{ccccc}
\toprule
                           & \multirow{2}{*}{Hyperparameter} & Simulation        & \multicolumn{2}{c}{Experiments}            \\ \cline{3-5} 
                           &                                 & Moderate-Scale    & Moderate-Scale         & Large-Scale       \\ \cline{2-5} 
\multirow{9}{*}{Distributed LB}       & Learning rate                   & $3\times 10^{-4}$ & $1\times 10^{-3}$      & $1\times 10^{-3}$ \\ \cline{2-5} 
                           & Batch size                      & $25$              & $25$                   & $12$              \\ \cline{2-5} 
                           & Hidden dimension                & $64$              & $64$                   & $128$              \\ \cline{2-5} 
                           & Hypernet dimension              & $32$              & $32$                   & $64$              \\ \cline{2-5} 
                           & Replay buffer size              & $3000$            & $3000$                 & $3000$            \\ \cline{2-5} 
                           & Episodes                        & $500$             & $120$                  & $200$              \\ \cline{2-5} 
                           & Updates per episode             & $10$              & $10$                   & $10$              \\ \cline{2-5} 
                           & Step interval                   & $0.5$s            & $0.25$s                & $0.25$s           \\ \cline{2-5} 
                           & Target entropy                  & $-|\mathcal{A}|$  & $-|\mathcal{A}|$       & $-|\mathcal{A}|$  \\ \midrule
\multirow{9}{*}{LB System} & TCT Distribution                & Exponential       & Real-world trace       & Real-world trace  \\ \cline{2-5} 
                           & Average TCT                     & $1$s              & $200$ms                & $200$ms           \\ \cline{2-5} 
                           & Average bytes per task & -              & $12$KiB                & $12$KiB           \\ \cline{2-5} 
                           & Traffic rate                    & $20.28$tasks/s    & $[650, 800]$tasks/s    & $2000$tasks/s           \\ \cline{2-5} 
                           & Number of LB agents             & $2$    			 & $2$    			      & $6$           \\ \cline{2-5} 
                           & Total number of servers         & $8$    			 & $7$                    & $20$           \\ \cline{2-5} 
                           & Server group 2       & $4$ (1-CPU) & $3$ (2-CPU)    & $10$ (2-CPU)           \\ \cline{2-5} 
                           & Server group 1       & $4$ (2-CPU) & $4$ (4-CPU)    & $10$ (4-CPU)           \\ \cline{2-5} 
                           & Episode duration                & $60$s             & $60$s                  & $60$s             \\ \bottomrule
\end{tabular}
}
\end{table}

MARL-based load balancing methods are trained in both simulator, and moderate- and large-scale testbed setups for various amount of episodes.
At the end of each episode, the RL models are trained and updated for $10$ iterations.
Given the total provisioned computational resource, the traffic rates of network traces for training are carefully selected so that the RL models can learn from sensitive cases where workloads should be carefully placed to avoid overloaded less powerful servers. 
The set of hyper-parameters are listed in Table~\ref{tab:rl_params}.

\section{Results}
\label{app:results}

\subsection{Inaccurate Server Weights}
\label{app:results-simulation-weights}

\begin{table}[t]
	\centering
	\caption{Four configurations with different application types.}
	\begin{tabular}{ccccc}
		\hline
		\multicolumn{1}{|c|}{\begin{tabular}[c]{@{}c@{}}Application\\Type\end{tabular}} & \multicolumn{1}{c|}{\begin{tabular}[c]{@{}c@{}}Pure\\CPU\end{tabular}} & \multicolumn{1}{c|}{\begin{tabular}[c]{@{}c@{}}CPU\\Intensive\end{tabular}} & \multicolumn{1}{c|}{\begin{tabular}[c]{@{}c@{}}Balanced\end{tabular}} & \multicolumn{1}{c|}{\begin{tabular}[c]{@{}c@{}}IO\\Intensive\end{tabular}} \\ \hline
		\multicolumn{1}{|c|}{\begin{tabular}[c]{@{}c@{}}Avg. CPU Time (s)\end{tabular}}                                                   & \multicolumn{1}{c|}{$1.$}  & \multicolumn{1}{c|}{$0.75$}  & \multicolumn{1}{c|}{$0.5$}  & \multicolumn{1}{c|}{$0.25$}  \\ \hline
		\multicolumn{1}{|c|}{\begin{tabular}[c]{@{}c@{}}Avg. IO Time (s)\end{tabular}}                                                   & \multicolumn{1}{c|}{$0.$}  & \multicolumn{1}{c|}{$0.25$}  & \multicolumn{1}{c|}{$0.5$}  & \multicolumn{1}{c|}{$0.75$}  \\ \hline
	\end{tabular}
	\label{tab:simulator-multi-stage}
\end{table}

In real-world systems, not only error-prone configurations, but also different application profiles can lead to inaccurate server weight assignments.
Using a similar setup where $2$ cluster of $4$ servers have the same IO processing speed but $2$x different CPU processing speeds, different application profiles are compared to derive the actual server processing capacity differences.
A $3$-stage application whose queries follow CPU-IO-CPU processing stages is compared with a pure CPU application.
Both CPU and IO processing time follow exponential distributions and the aggregated average FCT is $1$s.
The four different types of network applications are configured as in Table~\ref{tab:simulator-multi-stage}.
As depicted in Fig.~\ref{fig:motivation-multi-stage}, with different provisioned resource ratios for CPU (\texttt{2x}) and IO (\texttt{1x}) queues, to guarantee the optimal workload distribution fairness and make each server have the minised maximal remaining time to finish among all servers at all time, the weights to be assigned for servers with different provisioned capacities are stochastic and depend on different application profiles.
Therefore, it is a sub-optimal solution for existing load balancing algorithms to statically configure server weights based on computational resources.

The setup in the paper for Table~\ref{tab:simulation-optimal-distance} is the following. There are $2$ LB agents and $8$ servers. $4$ servers have $1$ CPU worker-thread each while the other $4$ servers have $2$ CPU worker-threads each, to simulate the different server processing capacities.
Three types of applications are compared. $100\%$-CPU application is a single stage application, whose expected time to process is $1$s in the CPU queue and $0$s in the IO queue.
$75\%$-CPU$+25\%$-IO application is a two stage application, whose expected time to process is $0.75$s in the CPU queue and $0.25$s in the IO queue, simulating the CPU-intensive applications.
$50\%$-CPU$+50\%$-IO application is a two stage application, whose expected time to process is $0.5$s in both the CPU and IO queue.
The actual time to process for each task follows an exponential distribution.
The traffic rate is normalised to consume on average $84.5\%$ resources.

\subsection{Ablation Results}
\label{app:results-ablation}

Besides the experiments conducted in the paper, we further study the following aspects of the application of MARL on real-world network load balancing problems.

\subsubsection{Reward Engineering}
\label{app:results-reward-engineering}

To verify the effectiveness of the proposed potential function VBF, we compare it with a set of different reward functions, including makespan (MS), PBF, and coefficient of variation (CV).
During our study based on real-world testbed, we found that, when using VBF as the reward, the convergence is fast at the beginning of the training process and the sample variance of average flow duration (as an estimation of the queuing and processing delay) on each server becomes close to zero.
However, it does not necessarily mean that the load balancing policy is optimal and the NE is achieved.
To capture the small variance which is close-to-zero, we also calculate the logarithm of VBF ($\log$VBF) as reward.
And the combination of VBF + $\log$VBF is an empirical design aiming at faster convergence towards the NE policy.
The complete comparison results are shown in Table~\ref{tab:compare_small_scale_full_avg} (average QoS) and in Table~\ref{tab:compare_small_scale_full_99} ($99$th percentile QoS), where the proposed distributed MARL framework achieves the best performance for most cases.
To provide a complete view of all comparison results besides the one shown in Fig.~\ref{fig:eval-cdf}, we show the CDF of task completion time under all test cases in Fig.~\ref{fig:app-cdf}
Accompanying the evaluation results of average QoS in large-scale testbed in Table~\ref{tab:compare_large_scale}, we also show in Table~\ref{tab:compare_large_scale_99} the $99$th percentile QoS in large-scale testbed.

\begin{figure}[tbp]
	\centering
	\begin{subfigure}{0.45\columnwidth}
		\centering
		\includegraphics[height=1in]{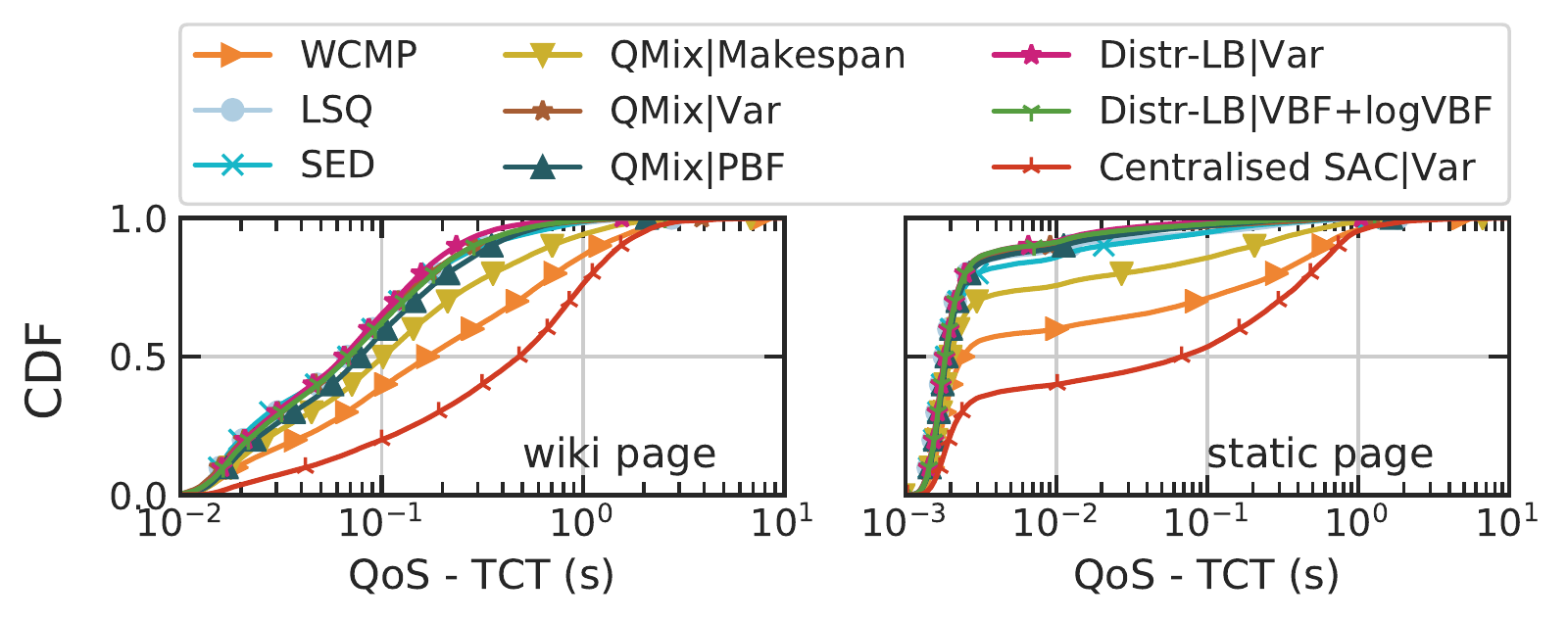}
		\vskip -.1in
		\caption{ }
		\label{fig:app-cdf-hour0}
	\end{subfigure}
	\hspace{.05in}
	\begin{subfigure}{0.45\columnwidth}
		\centering
		\includegraphics[height=1in]{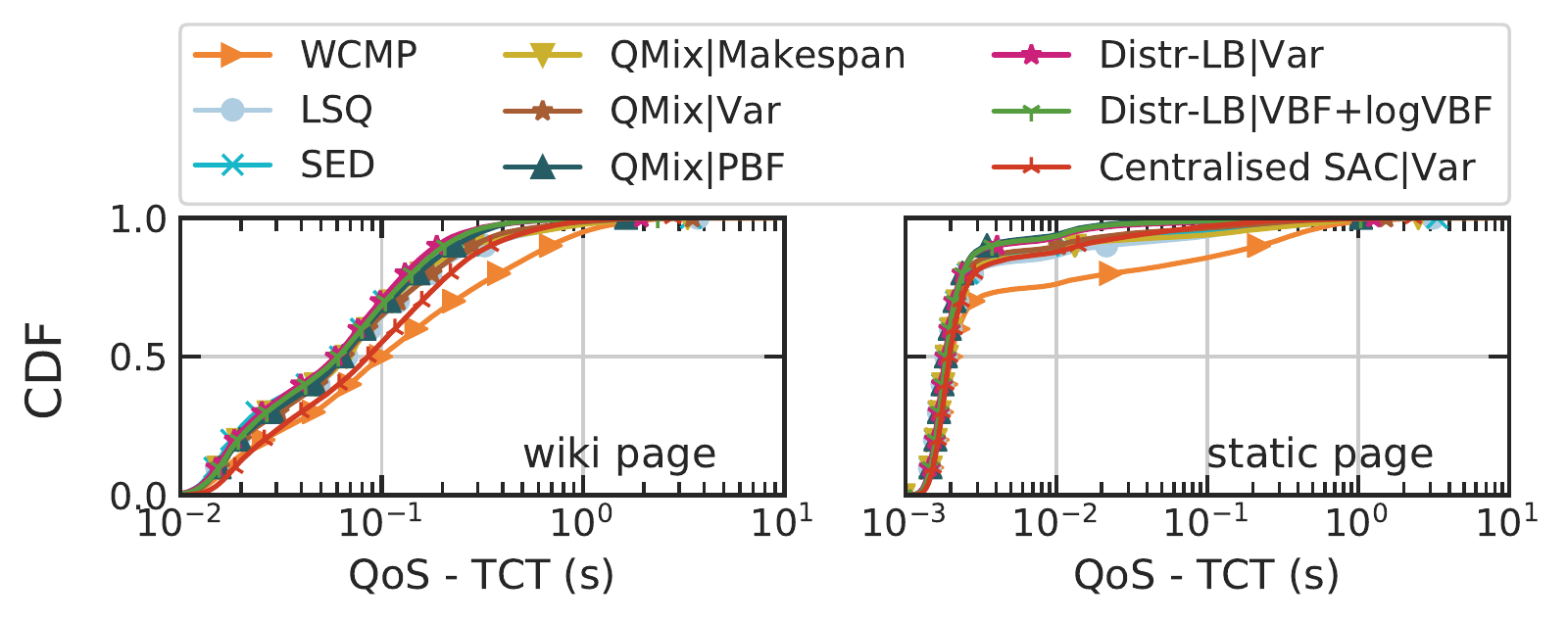}
		\vskip -.1in
		\caption{ }
		\label{fig:app-cdf-hour2}
	\end{subfigure}
    \begin{subfigure}{0.45\columnwidth}
		\centering
		\includegraphics[height=1in]{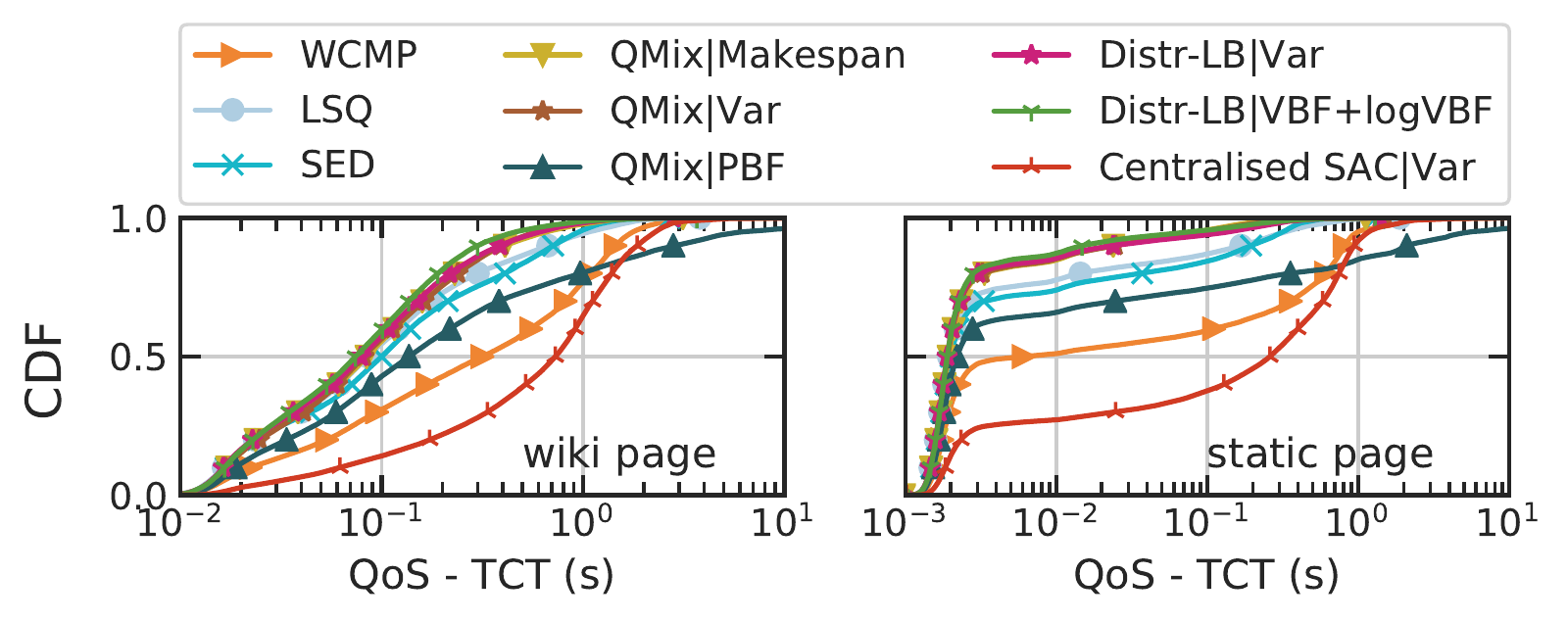}
		\vskip -.1in
		\caption{ }
		\label{fig:app-cdf-hour4}
	\end{subfigure}
	\hspace{.05in}
	\begin{subfigure}{0.45\columnwidth}
		\centering
		\includegraphics[height=1in]{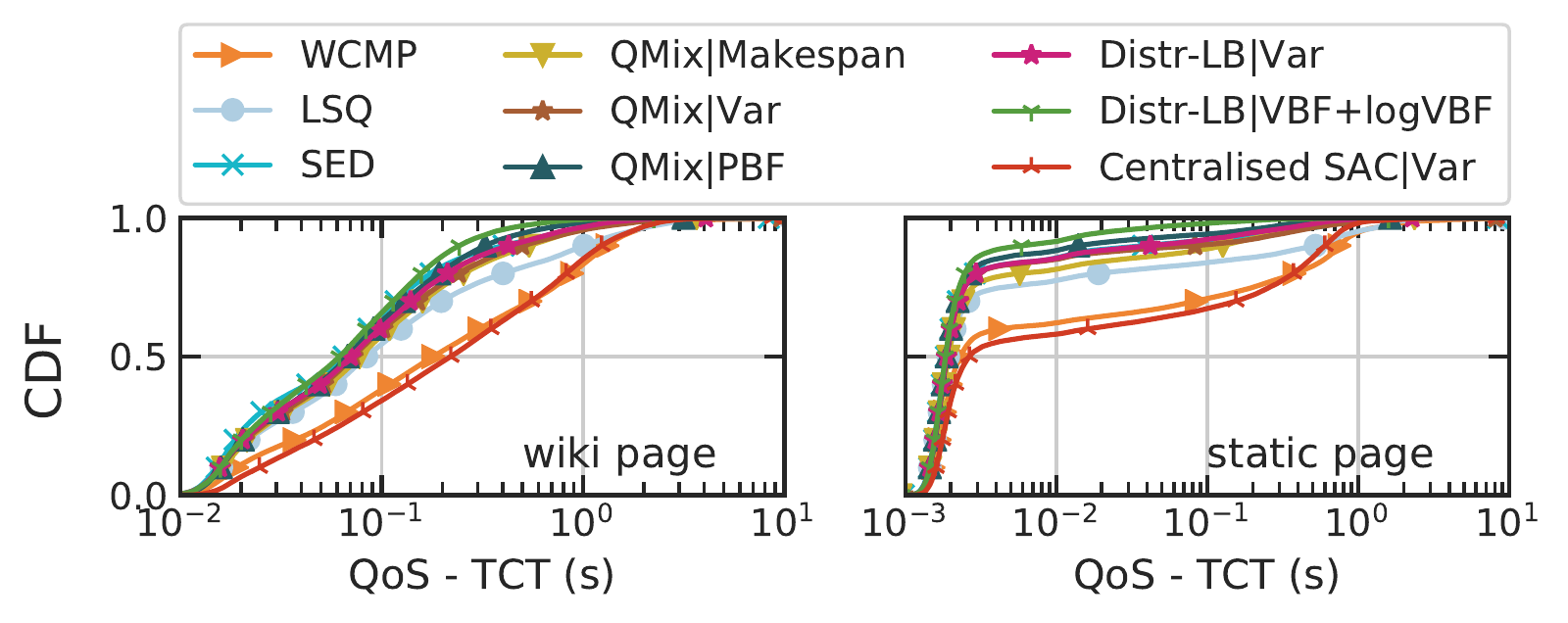}
		\vskip -.1in
		\caption{ }
		\label{fig:app-cdf-hour18}
	\end{subfigure}
	\begin{subfigure}{0.45\columnwidth}
		\centering
		\includegraphics[height=1in]{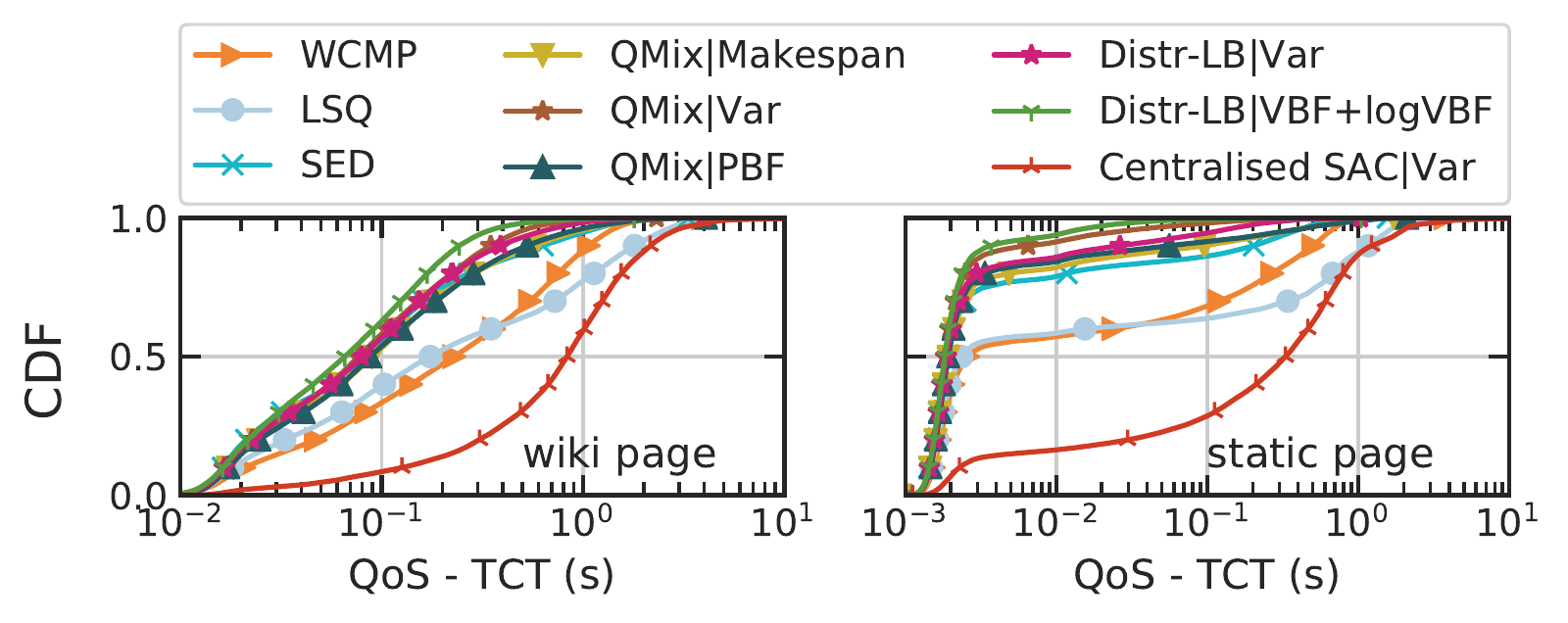}
		\vskip -.1in
		\caption{ }
		\label{fig:app-cdf-hour20}
	\end{subfigure}
	\hspace{.05in}
	\begin{subfigure}{0.45\columnwidth}
		\centering
		\includegraphics[height=1in]{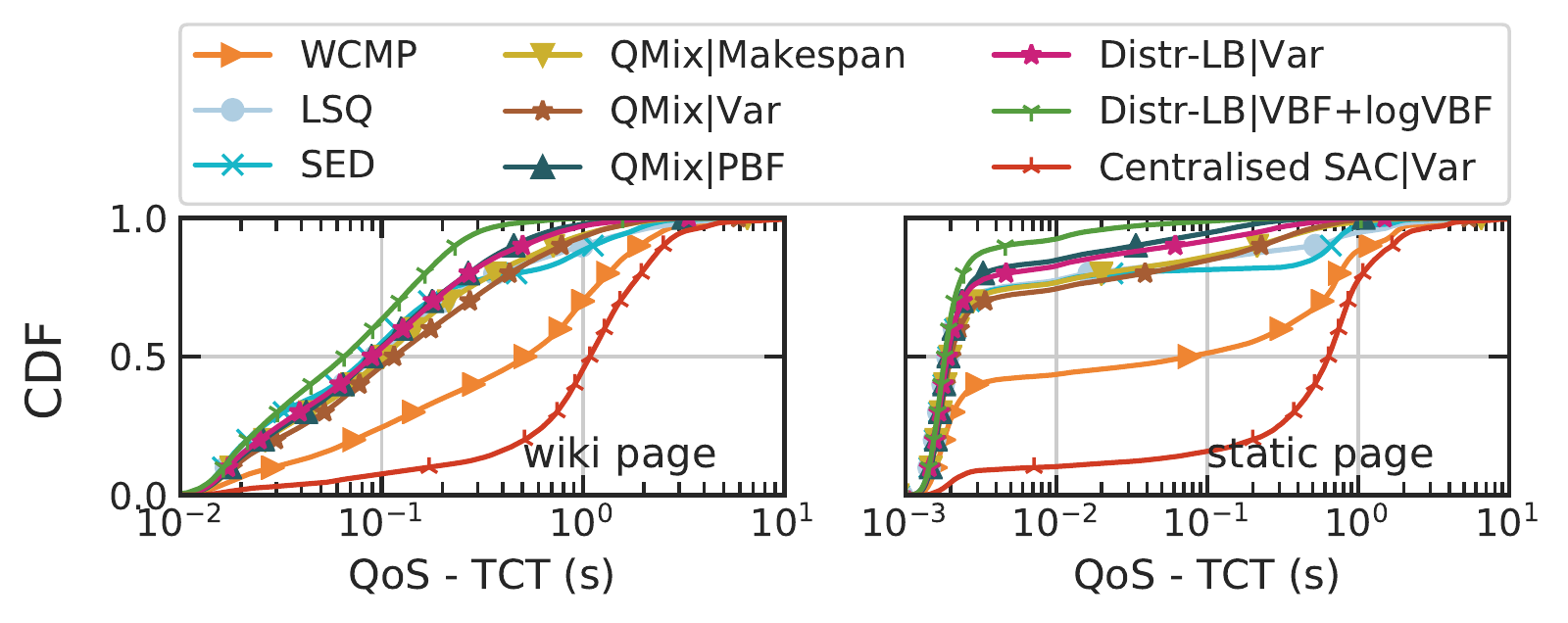}
		\vskip -.1in
		\caption{ }
		\label{fig:app-cdf-hour22}
	\end{subfigure}
	\caption{Experimental results with real-world network traces from different period of time during a day, which demonstrates the effectiveness of the proposed distributed RL framework with VBF as rewards.}
	\label{fig:app-cdf}
	\vskip -.2in
\end{figure}

\begin{table}[tbp]
    \centering
    \caption{Complete results of \emph{average} QoS (s)
    for comparison in moderate-scale real-world network setup (DC network and traffic).}
    \resizebox{\columnwidth}{!}{ 
    \begin{tabular}{c|c|c|c|c|c|c|c|c|c}
    \toprule
    \multicolumn{2}{c|}{\multirow{2}{*}{Method}} & \multicolumn{2}{c|}{Period I ($796.315$ queries/s)} & \multicolumn{2}{c|}{Period II ($687.447$ queries/s)} & \multicolumn{2}{c|}{Period III ($784.522$ queries/s)} & \multicolumn{2}{c}{Period IV ($784.522$ queries/s)} \\
    \cline{3-10}
    \multicolumn{2}{c|}{} &  \multicolumn{1}{c|}{Wiki} & \multicolumn{1}{c|}{Static} & \multicolumn{1}{c|}{Wiki} & \multicolumn{1}{c|}{Static} &\multicolumn{1}{c|}{Wiki} & \multicolumn{1}{c|}{Static} & \multicolumn{1}{c|}{Wiki} & \multicolumn{1}{c}{Static}  \\
    \hline
      \multicolumn{2}{c|}{WCMP} & $0.435\pm0.083 $ & $0.171\pm0.055 $ & $0.254\pm0.087 $ & $0.073\pm0.056 $ & $0.412\pm0.101$ & $0.134\pm0.059$ & $ 0.834\pm0.323$ & $0.492\pm0.276 $ \\
      \multicolumn{2}{c|}{LSQ} & $ 0.141\pm0.073$ & $0.023\pm0.030 $ & $ 0.143\pm0.040$ & $0.023\pm0.011$ & $0.620\pm0.442$ & $0.339\pm0.316$ & $0.357\pm0.373 $ & $0.173\pm0.299 $  \\
      \multicolumn{2}{c|}{SED}  & $0.137\pm0.076$ & $0.020\pm0.023 $ & $ 0.131\pm0.067$ & $0.027\pm0.035 $ &  $0.215\pm0.210$ & $0.051\pm0.081$ & $0.346\pm0.496$ &$0.169\pm0.330 $ \\ \hline
      \multirow{2}{*}{\textbf{RLB-SAC}~\cite{yao2022reinforced}} 
      & Jain & $0.137\pm0.020$ & $0.009\pm0.006$ & $0.125\pm0.035$ & $0.012\pm0.008$ & $0.193\pm0.073$ & $0.026\pm0.022$ & $0.204\pm0.084$ & $0.039\pm0.047$ \\
      & G    & $0.140\pm0.053$ & $0.015\pm0.019$ & $0.103\pm0.022$ & $0.010\pm0.007$ & $0.149\pm0.049$ & $0.015\pm0.011$ & $0.155\pm0.052$ & $0.011\pm0.011$ \\	\hline
      \multirow{7}{*}{\textbf{QMix-LB}} 
      & MS & $0.258\pm0.174 $ & $0.071\pm0.087 $ & $0.142\pm0.073 $ & $0.030\pm0.034 $ & $0.217\pm0.157$ & $0.048\pm0.069$ &$0.263\pm0.202 $ &$0.073\pm0.092 $  \\
     & $\log$MS & $0.167\pm0.031 $ & $0.009\pm0.004 $ & $0.132\pm0.034 $ & $0.011\pm0.008 $ & $0.844\pm1.376$ & $0.635\pm1.249$ & $ 0.278\pm0.130$ &$0.041\pm0.038 $  \\
      & VBF & $ 0.128\pm0.052$ & $0.014\pm0.017 $ & $0.132\pm0.075 $ & $0.016\pm0.025 $ & $0.141\pm0.025$ & $0.008\pm0.004$ &$ 0.286\pm0.162$ & $0.068\pm0.066 $ \\
     & $\log$VBF & $\mathbf{0.106\pm0.011} $ & $0.007\pm0.001 $ & $0.109\pm0.032 $ & $0.011\pm0.009 $ & $0.171\pm0.043$ & $0.022\pm0.013$ &$0.223\pm0.045 $ &$ 0.026\pm0.017$  \\
      & VBF+$\log$VBF & $0.112\pm0.005 $ & $\mathbf{0.005\pm0.002} $ & $ 0.101\pm0.010$ & $0.005\pm0.001 $ & $0.187\pm0.090$ & $0.024\pm0.029$ &$0.201\pm0.080 $ & $0.021\pm0.020 $ \\
      & PBF &  $0.142\pm0.035 $ & $0.012\pm0.006 $ & $0.099\pm0.011 $ & $\mathbf{0.004\pm0.001}$ &$0.211\pm0.153$ & $0.047\pm0.078$ & $0.181\pm0.042 $& $0.018\pm0.009 $ \\
      & CV & $ 0.407\pm0.505$ & $0.201\pm0.340 $ & $0.113\pm0.036 $ & $0.009\pm0.008 $ & $0.203\pm0.089$ & $0.039\pm0.037$ & $0.219\pm0.072 $& $ 0.031\pm0.017$\\ \hline
    
     \multirow{3}{*}{\textbf{Centr-LB}} 
      & VBF & $0.690\pm0.211 $ & $0.284\pm0.181 $ & $0.152\pm0.041 $ & $0.016\pm0.011 $ & $1.068\pm0.386$ & $0.570\pm0.378$ & $ 1.378\pm0.377$& $ 0.867\pm0.350$\\
     & $\log$VBF & $0.676\pm0.231 $ & $ 0.265\pm0.151$ & $0.160\pm0.023 $ & $ 0.013\pm0.005$ & $0.938\pm0.200$ & $0.446\pm0.179$ & $ 0.972\pm0.288$& $ 0.495\pm0.268$\\
     & VBF+$\log$VBF & $ 0.520\pm0.034$ & $0.167\pm0.017 $ & $0.192\pm0.040 $ & $0.019\pm0.014 $ & $0.759\pm0.254$ & $0.306\pm0.222$ & $ 1.013\pm0.168$&$ 0.520\pm0.167$ \\ \hline
     \multirow{4}{*}{\makecell{\textbf{Distr-LB}\\(this paper)}} 
      & VBF & $\mathbf{0.106\pm0.013} $ & $0.007\pm0.002 $ & $\mathbf{0.090\pm0.016} $ & $ 0.007\pm0.005$ & $0.159\pm0.054$ & $0.017\pm0.009$ & $0.196\pm0.091 $& $ 0.032\pm0.033$\\
     & $\log$VBF & $ 0.139\pm0.021$ & $0.011\pm0.004 $ & $0.129\pm0.032 $ & $ 0.012\pm0.011$ & $0.250\pm0.156$ & $0.057\pm0.077$ & $ 0.226\pm0.059$&$ 0.038\pm0.019$ \\
     & VBF+$\log$VBF & $ 0.126\pm0.038$ & $0.009\pm0.006 $ & $ 0.094\pm0.023$ & $ 0.006\pm0.006$ & $\mathbf{0.108\pm0.022}$ & $\mathbf{0.004\pm0.001}$ & $\mathbf{0.104\pm0.013} $&
    $\mathbf{0.006\pm0.003}$\\
     & CV & $0.150\pm0.040 $ & $0.011\pm0.009 $ & $0.149\pm0.060 $ & $0.026\pm0.025 $ & $0.301\pm0.146$ & $0.066\pm0.072$ & $ 0.267\pm0.156$& $0.051\pm0.052 $\\

    \bottomrule
    \end{tabular}
    }
    \label{tab:compare_small_scale_full_avg}
\end{table}

\begin{table}[tbp]
	\centering
	\caption{Complete results of \emph{99th percentile} QoS (s)
	for comparison in moderate-scale real-world network setup (DC network and traffic).}
	\resizebox{\columnwidth}{!}{ 
	\begin{tabular}{c|c|c|c|c|c|c|c|c|c}
	\toprule
	\multicolumn{2}{c|}{\multirow{2}{*}{Method}} & \multicolumn{2}{c|}{Period I ($796.315$ queries/s)} & \multicolumn{2}{c|}{Period II ($687.447$ queries/s)} & \multicolumn{2}{c|}{Period III ($784.522$ queries/s)} & \multicolumn{2}{c}{Period IV ($784.522$ queries/s)} \\
	\cline{3-10}
	\multicolumn{2}{c|}{} &  \multicolumn{1}{c|}{Wiki} & \multicolumn{1}{c|}{Static} & \multicolumn{1}{c|}{Wiki} & \multicolumn{1}{c|}{Static} &\multicolumn{1}{c|}{Wiki} & \multicolumn{1}{c|}{Static} & \multicolumn{1}{c|}{Wiki} & \multicolumn{1}{c}{Static}  \\
	\hline
	  \multicolumn{2}{c|}{WCMP} & $5.801 \pm 4.519$ & $4.462 \pm 3.867$   & $4.019 \pm 3.601$ & $3.192 \pm 3.543$   & $3.239 \pm 2.721$ & $2.305 \pm 2.700$   & $8.066 \pm 7.025$ & $6.733 \pm 5.329$ \\
	  \multicolumn{2}{c|}{LSQ}  & $0.722 \pm 0.487$ & $0.195 \pm 0.314$   & $0.814 \pm 0.478$ & $0.288 \pm 0.259$   & $1.846 \pm 1.915$ & $1.168 \pm 1.575$   & $1.257 \pm 1.921$ & $0.831 \pm 2.002$  \\
	  \multicolumn{2}{c|}{SED}  & $0.706 \pm 0.399$ & $0.208 \pm 0.246$   & $0.697 \pm 0.460$ & $0.217 \pm 0.291$   & $0.726 \pm 0.554$ & $0.203 \pm 0.261$   & $0.909 \pm 1.180$ & $0.450 \pm 1.112$  \\ \hline
	\multirow{2}{*}{\textbf{RLB-SAC}~\cite{yao2022reinforced}} 
	& Jain & $0.858 \pm 0.240$ & $0.159 \pm 0.125$ & $0.830 \pm 0.358$ & $0.227 \pm 0.186$ & $1.227 \pm 0.489$ & $0.354 \pm 0.246$ & $1.283 \pm 0.594$ & $0.408 \pm 0.374$ \\
	& G & $0.945 \pm 0.495$ & $0.185 \pm 0.214$ & $0.682 \pm 0.255$ & $0.177 \pm 0.162$ & $1.003 \pm 0.459$ & $0.225 \pm 0.176$ & $0.973 \pm 0.389$ & $0.166 \pm 0.156$ \\	\hline
	 \multirow{7}{*}{\textbf{QMix-LB}} 
	  & MS & $1.469 \pm 0.789$ & $0.584 \pm 0.547$   & $1.095 \pm 0.694$ & $0.444 \pm 0.423$   & $1.182 \pm 0.801$ & $0.420 \pm 0.483$   & $1.447 \pm 0.885$ & $0.751 \pm 0.772$ \\
	 & $\log$MS & $0.985 \pm 0.264$ & $0.117 \pm 0.043$   & $0.909 \pm 0.388$ & $0.172 \pm 0.142$   & $7.043 \pm 12.237$& $6.427 \pm 12.479$  & $1.326 \pm 0.584$ & $0.371 \pm 0.305$ \\
	  & VBF & $0.732 \pm 0.395$ & $0.159 \pm 0.239$   & $0.665 \pm 0.550$ & $0.157 \pm 0.278$   & $0.744 \pm 0.278$ & $0.123 \pm 0.093$   & $1.028 \pm 0.694$ & $0.279 \pm 0.365$ \\
	 & $\log$VBF & $0.682 \pm 0.100$ & $0.124 \pm 0.019$   & $0.772 \pm 0.313$ & $0.205 \pm 0.159$   & $1.174 \pm 0.323$ & $0.382 \pm 0.183$   & $1.426 \pm 0.323$ & $0.327 \pm 0.153$ \\
	  & VBF+$\log$VBF & $0.664 \pm 0.057$ & $\mathbf{0.087 \pm 0.056}$& $0.611 \pm 0.097$ & $0.055 \pm 0.027$   & $1.171 \pm 0.568$ & $0.302 \pm 0.293$   & $1.206 \pm 0.501$ & $0.278 \pm 0.239$ \\
	  & PBF & $0.661 \pm 0.193$ & $\mathbf{0.087 \pm 0.099}$ & $0.505 \pm 0.119$ & $0.048 \pm 0.029$   & $0.768 \pm 0.728$ & $0.205 \pm 0.465$   & $0.726 \pm 0.433$ & $0.128 \pm 0.136$ \\
	  & CV & $1.928 \pm 2.228$ & $1.281 \pm 2.095$   & $0.708 \pm 0.405$ & $0.131 \pm 0.130$   & $1.331 \pm 0.593$ & $0.481 \pm 0.297$   & $1.344 \pm 0.329$ & $0.451 \pm 0.218$ \\ \hline	
	  \multirow{3}{*}{\textbf{Centr-LB}} 
	 & VBF & $3.101 \pm 1.582$ & $1.985 \pm 1.790$   & $0.903 \pm 0.350$ & $0.328 \pm 0.353$   & $4.409 \pm 2.693$ & $3.629 \pm 3.219$   & $6.649 \pm 4.562$ & $6.120 \pm 4.721$ \\
	 & $\log$VBF & $2.715 \pm 0.444$ & $1.718 \pm 0.547$   & $1.016 \pm 0.229$ & $0.264 \pm 0.092$   & $3.247 \pm 0.725$ & $2.136 \pm 0.832$   & $4.286 \pm 2.091$ & $3.459 \pm 2.323$ \\
	 & VBF+$\log$VBF & $2.459 \pm 0.101$ & $1.309 \pm 0.063$   & $1.243 \pm 0.358$ & $0.285 \pm 0.189$   & $2.796 \pm 0.900$ & $1.702 \pm 1.287$   & $3.466 \pm 0.820$ & $2.628 \pm 1.142$ \\ \hline	
	 \multirow{4}{*}{\makecell{\textbf{Distr-LB}\\(this paper)}} 
	  & VBF & $\mathbf{0.651 \pm 0.151}$ & $0.119 \pm 0.072$   & $0.571 \pm 0.237$ & $0.133 \pm 0.136$   & $1.039 \pm 0.302$ & $0.298 \pm 0.125$   & $1.187 \pm 0.594$ & $0.355 \pm 0.318$ \\
	 & $\log$VBF & $0.923 \pm 0.162$ & $0.193 \pm 0.086$   & $0.933 \pm 0.415$ & $0.243 \pm 0.302$   & $1.491 \pm 0.764$ & $0.579 \pm 0.531$   & $1.481 \pm 0.473$ & $0.558 \pm 0.286$ \\
	 & VBF+$\log$VBF & $0.745 \pm 0.316$ & $0.185 \pm 0.152$   & $\mathbf{0.385 \pm 0.094}$& $\mathbf{0.023 \pm 0.003}$& $\mathbf{0.595 \pm 0.199}$& $\mathbf{0.051 \pm 0.030}$& $\mathbf{0.563 \pm 0.180}$& $\mathbf{0.100 \pm 0.073}$  \\
	 & CV & $0.865 \pm 0.261$ & $0.147 \pm 0.121$   & $1.109 \pm 0.668$ & $0.433 \pm 0.431$   & $1.730 \pm 0.468$ & $0.612 \pm 0.420$   & $1.383 \pm 0.666$ & $0.446 \pm 0.345$ \\	
	\bottomrule
	\end{tabular}
	}
	\label{tab:compare_small_scale_full_99}
\end{table}

\begin{table}[t]
    \vskip -.2in
    \scriptsize
    \centering
    \caption{Comparison of 99th percentile QoS (s) in large-scale real-world network setup (DC network and traffic).}
    \begin{tabular}{c|c|c|c|c|c}
    \toprule
    \multicolumn{2}{c}{\multirow{2}{*}{Method}} &  \multicolumn{2}{|c|}{Period I ($2022.855$ queries/s)} & \multicolumn{2}{c}{Period II ($2071.129$ queries/s)} \\
    \cline{3-6}
    \multicolumn{2}{c|}{} & \multicolumn{1}{c|}{Wiki} & \multicolumn{1}{c|}{Static} & \multicolumn{1}{c|}{Wiki} & \multicolumn{1}{c}{Static} \\
    \hline
     \multicolumn{2}{c|}{WCMP}  &    $3.014 \pm 0.612$  & $2.152 \pm 0.907$ & $4.290 \pm 3.593$ & $3.300 \pm 3.308$ \\
      \multicolumn{2}{c|}{LSQ}  &    $1.863 \pm 0.888$  & $0.843 \pm 0.773$ & $1.243 \pm 1.389$ & $0.675 \pm 1.223$ \\
      \multicolumn{2}{c|}{SED}  &    $0.891 \pm 0.475$  & $0.208 \pm 0.251$ & $1.074 \pm 0.751$ & $0.592 \pm 0.650$ \\
      \multicolumn{2}{c|}{RLB-SAC-G\cite{yao2022reinforced}}  &   $1.064 \pm 0.283$  & $0.210 \pm 0.132$ & $0.739 \pm 0.317$ & $0.186 \pm 0.214$ \\ \cline{1-2}
    \multirow{2}{*}{\textbf{QMix-LB}} 
    & VBF                       &   $1.104 \pm 0.481$  & $0.241 \pm 0.264$ & $1.223 \pm 1.169$ & $0.634 \pm 0.983$ \\
    & PBF                       &   $1.201 \pm 0.321$  & $0.196 \pm 0.112$ & $0.583 \pm 0.103$ & $0.071 \pm 0.050$ \\ \cline{1-2}
    \multirow{2}{*}{\makecell{\textbf{Distr-LB}\\(this paper)}} 
    & VBF                       &   $1.350 \pm 0.311$  & $0.263 \pm 0.139$ & $1.180 \pm 0.702$ & $0.448 \pm 0.371$ \\
    & VBF+$\log$VBF             & $\mathbf{0.890 \pm 0.250}$ & $\mathbf{0.103 \pm 0.064}$ & $\mathbf{0.531 \pm 0.149}$ & $\mathbf{0.057 \pm 0.039}$ \\
    \bottomrule
    \end{tabular}
    \label{tab:compare_large_scale_99}
\end{table}

\subsubsection{Communication Overhead of CTDE and Centralised RL}
\label{app:results-ablation-comm}

This section studies the communication overhead of CTDE RL scheme and analyses its impact on real-world distributed systems.

First, we discuss the communication overhead in data center networks in two-folds: throughput and latency.

\begin{figure}[t]
	\centering
	\centerline{\includegraphics[width=.6\columnwidth]{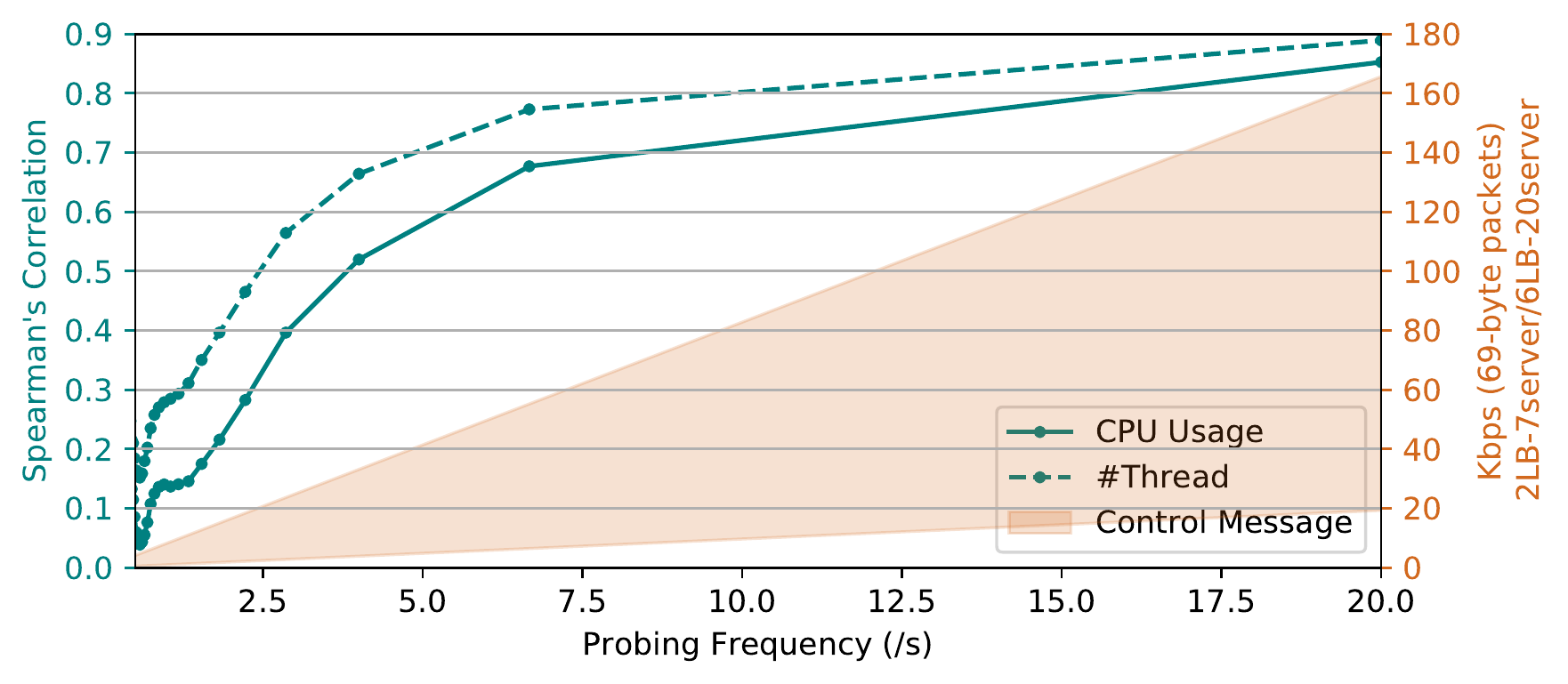}}
	\vskip -.05in
	\caption{Correlation (Spearman) increases when the probing frequency grows, yet, so do additional control messages.}
	\label{fig:app-ablation-comm-corr}
	\vskip -.2in
\end{figure}

\begin{figure}[tbp]
	\centering
	\begin{subfigure}{0.4\columnwidth}
		\centering
		\includegraphics[height=1.3in]{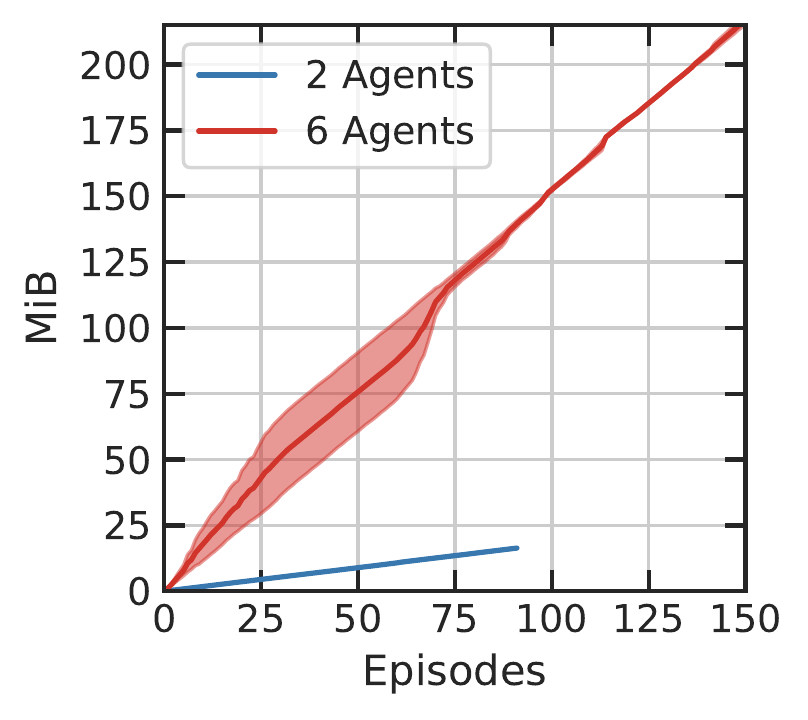}
		\vskip -.1in
		\caption{ }
		\label{fig:app-overhead-size}
	\end{subfigure}
	\hspace{.05in}
	\begin{subfigure}{0.55\columnwidth}
		\centering
		\includegraphics[height=1.3in]{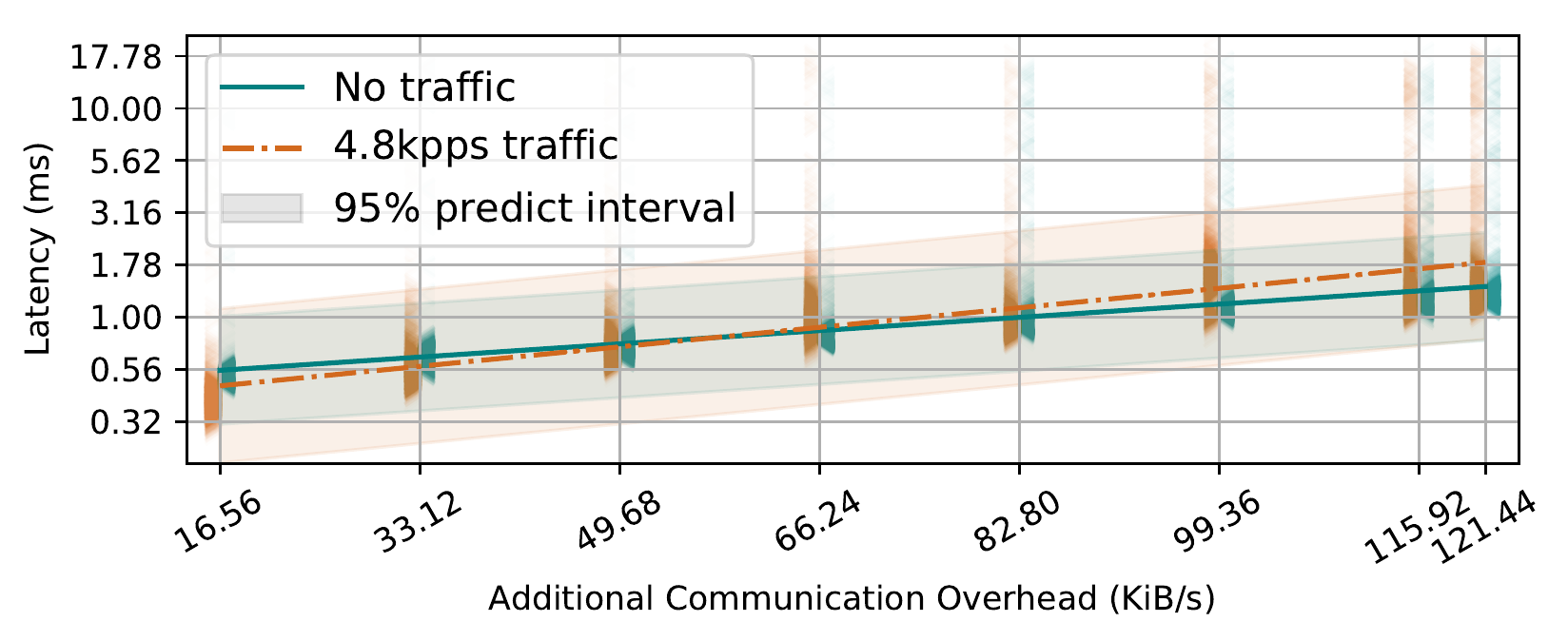}
		\vskip -.1in
		\caption{ }
		\label{fig:app-overhead-compare}
	\end{subfigure}
	\caption{Communication overhead for CTDE (a) grows linearly during training and (b) can have negative effects on the packet transmission latency of the whole networking system.}
	\label{fig:app-overhead}
\end{figure}

\begin{figure}[t]
	\centering
	\centerline{\includegraphics[width=.6\columnwidth]{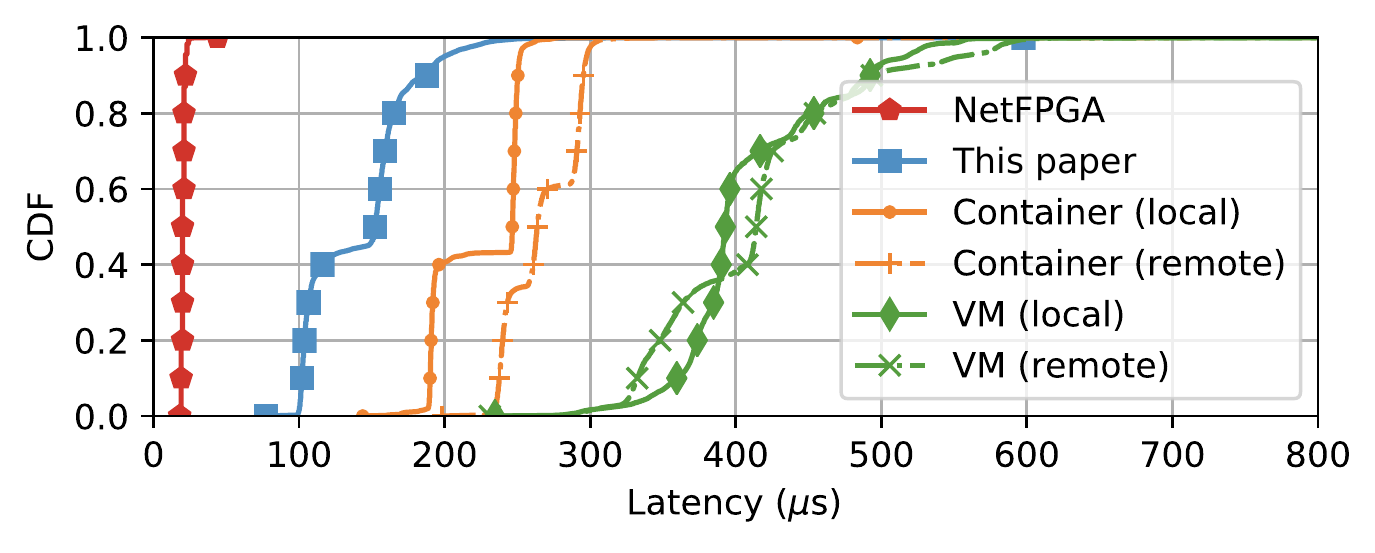}}
	\vskip -.05in
	\caption{Feature collection latency comparison against active probing techniques.}
	\label{fig:app-ablation-comm-collect}
\end{figure}

\begin{enumerate}[wide, labelwidth=!, labelindent=0pt]
	\item \textbf{Thoughput:} Active signaling (\eg periodically probing, or sharing messages) is an intrinsic way to observe and measure system states so that informed decisions can be made to improve performance~\cite{ctde2019, ctde2020nips}.
	Higher communication frequency gives more relevant and timely observations yet there is a trade-off between communication frequency and additionally consumed bandwidth.
	Especially, in large distributed systems like data center networks, services are organized by multiple server clusters scattered all over the physical data center network in the era of cloud computing.
	Thus, management traffic among different nodes can cascade and plunder the bandwidth for data transmission in high-tier links.
	To demonstrate the trade-off between measurement quality and throughput overhead, we have conducted experiments to evaluate (i) the relevance of collected server utilization information to the actual server utilization information with root mean square error (RMSE) and Spearman’s Correlation in our testbed on physical servers.
	When a controller VM periodically probes a server cluster via TCP sockets\footnote{In the $69$-byte control packet emitted by the server, the $24$-byte payload consists of the server ID, CPU and memory usage, and the number of busy application threads.}, as depicted in Fig.~\ref{fig:app-ablation-comm-corr}, the visibility over system states (relevance between measurements and ground truth) correlates with the probing frequency.
	Additional management traffic within a single service cluster —- behind one virtual IP (VIP) —- can exceed the $90$th percentile of per-destination-rack flow rate ($100$kbps as depicted in Figure 8a in~\cite{facebook-dc-traffic}) in Facebook data center in production. 

	As depicted in Fig.~\ref{fig:app-overhead-size}, CTDE RL scheme requires agents to communicate and share their trajectories, which include the observed states and actions.
	This leads to linearly increasing replay buffer size with the growth of number of episodes.
	The replay buffer size also grows with the number of agents which makes CTDE RL scheme not a scalable mechanism.
	Transmitting and synchronising replay buffer among agents incur additional communication overhead in the networking system, reducing the throughput for data transmission channel -- which can break full-bisection bandwidth (an important throughput related performance metric) in data center networks~\cite{zhang2018load} -- thus decreasing the QoS.

	\item \textbf{Latency:} Using the same network topology as the moderate-scale real-world testbed, when a single controller VM periodically transmit different amount of bytes via TCP sockets towards the agents, the latency overhead increases with the number of servers, which diminishes the QoS, as depicted in Fig.~\ref{fig:app-overhead-compare}.
	It is measured for per-packet round trip time (RTT) between to directly connected network nodes.
	While normal RTT is $0.099ms \pm 0.014ms$ in such setup, with additional communication overhead, RTT can grow more than $10$x.
	This is not considered as low additional latency, especially not in high performance networking systems.
	In elastic and cloud computing context and real-world setups, load balancers can be deployed in different racks~\cite{gandhi2014duet}.
	There can be multiple hops between two nodes and one connection consists of tens of hundreds of packets, which can lead to cascaded high latency.

	Based on the analysis of Fig.~\ref{fig:app-ablation-comm-corr}, we can see that delayed measurement and communication can cause degraded system state observation.
	To further demonstrate the performance of the passive feature collection mechanism which incurs absolutely zero communication overhead, an additional experiment is conducted to compare the feature collection latency.
	The latency overhead of passive feature collection process in our paper using POSIX shared memory is compared with different active probing techniques.
	The idle communication latency is compared using both KVM and Docker containers between two hosts either deployed on the same machine (local) or on two neighbor machines (remote).
	To compare with the shortest latency possible of an hardware-based SDN controller directly connected to the agent, a loopback test is conducted using a NetFPGA~\cite{zilberman2014netfpga} connected to the machine via both Ethernet and PCIe.
	We parse features stored in the local shared memory with a simple Python script without generating control messages.
	As depicted in Fig.~\ref{fig:app-ablation-comm-collect}, its median processing latency outperforms typical VM- and container-based VNF probing mechanisms~\cite{nfvsdnsurvey, osm2019, opnfv2019, openstack} by more than $94.18 \mu s$.
\end{enumerate}

To evaluate the performance of the proposed algorithm within the operational range of traffic rates, we conducted the scaling experiment using $6$ LB agents and $20$ servers in the real-world testbed with traffic rates that range from low to high.
As shown in Table~\ref{tab:avg_wiki} and~\ref{tab:avg_static}, similar to the QoS ($99$-th percentile of the task completion time) evaluation in Table~\ref{tab:99_qos_wiki} and~\ref{tab:99_qos_static}.
Low traffic rates do not saturate server processing capacities and the servers are not stressed.
Therefore, all servers are able to handle all the requests without accumulating jobs in the queue regardless of the differences of their processing capacities.
However, under heavy traffic rates, LSQ still distribute workloads so as to maintain the same queue lengths on servers with different processing speeds, which leads to degraded average task completion time.
SED assigns more jobs in proportional to the number of CPUs deployed for each server, achieving slightly better performance than LSQ in terms of the average task completion time.
The proposed Distr-LB outperforms both LSQ and SED especially under heavy traffic rates, thus when servers undergo heavy resource utilisation.
Since the server processing speed for different applications is not necessarily proportional to the number of CPU--as we have discussed over Fig.~\ref{fig:motivation-multi-stage} in Sec.~\ref{sec:background}, Distr-LB is able to learn the appropriate ratio of workload distribution for servers with different capacities.

\begin{table}[tbp]
\scriptsize
\centering
\caption{Comparison of average job completion time (s) of static pages under different traffic rates using large-scale real-world setup.}
\begin{tabularx}{\textwidth}{c|c|X|X|X|X|X|X|X|X|X}
\toprule
 \multicolumn{2}{c|}{\multirow{2}{*}{Method}} & \multicolumn{9}{c}{Traffic Rate (queries/s)}  \\ \cline{3-11}
 \multicolumn{2}{c|}{} & 731.534 & 1097.3  & 1463.067 & 1828.834  & 2194.601 & 2377.484  & 2560.368 & 2743.251  & 2926.135\\ \hline
 \multicolumn{2}{c|}{\multirow{2}{*}{LSQ}} & 0.048\newline$\pm$0.002  & 0.055\newline$\pm$0.003  & 0.059\newline$\pm$0.003 & 0.069\newline$\pm$0.008  & 0.131\newline$\pm$0.070  & 0.643\newline$\pm$0.325 & 1.910\newline$\pm$0.269  & 2.873\newline$\pm$0.215 & 3.545\newline$\pm$0.146\\ \hline
 \multicolumn{2}{c|}{\multirow{2}{*}{SED}}  & 0.054\newline$\pm$0.001 & 0.061\newline$\pm$0.004  & 0.068\newline$\pm$0.004 & 0.080\newline$\pm$0.004  & 0.117\newline$\pm$0.025  & 0.660\newline$\pm$0.396 &  1.718\newline$\pm$0.366 & 2.767\newline$\pm$0.207 & 3.482\newline$\pm$0.189 \\ \hline
 \multirow{4}{*}{\makecell{\textbf{Distr-LB}\\(this paper)}}  &  \multirow{2}{*}{VBF} & 0.047\newline$\pm$0.001  & 0.054\newline$\pm$0.003 & 0.059\newline$\pm$0.005  & \textbf{0.066\newline$\pm$0.007}  & 0.105\newline$\pm$0.035  & \textbf{0.266\newline$\pm$0.139} & 1.465\newline$\pm$0.115 & 2.047\newline$\pm$0.145 & 2.704\newline$\pm$0.108 \\ \cline{2-11}
 & \multirow{2}{*}{VBF+$\log$VBF}  &  0.047\newline$\pm$0.001 & 0.054\newline$\pm$0.004  & 0.059\newline$\pm$0.004 & 0.069\newline$\pm$0.008  & \textbf{0.084\newline$\pm$0.009}  & 0.413\newline$\pm$0.249  &  \textbf{1.183\newline$\pm$0.063} & \textbf{1.838\newline$\pm$0.083} & \textbf{2.513\newline$\pm$0.105} \\ 
\bottomrule
\end{tabularx}
\label{tab:avg_wiki}
\end{table}

\begin{table}[tbp]
\scriptsize
\centering
\caption{Comparison of average job completion time (s) of static pages under different traffic rates using large-scale real-world setup.}
\begin{tabularx}{\textwidth}{c|c|X|X|X|X|X|X|X|X|X}
\toprule
 \multicolumn{2}{c|}{\multirow{2}{*}{Method}} & \multicolumn{9}{c}{Traffic Rate (queries/s)}  \\ \cline{3-11}
 \multicolumn{2}{c|}{} & 731.534 & 1097.3  & 1463.067 & 1828.834  & 2194.601 & 2377.484  & 2560.368 & 2743.251  & 2926.135\\ \hline
 \multicolumn{2}{c|}{\multirow{2}{*}{LSQ}} & 0.004\newline$\pm$0.001  & 0.004\newline$\pm$0.000  & 0.003\newline$\pm$0.000 & 0.004\newline$\pm$0.000  & 0.018\newline$\pm$0.023  & 0.252\newline$\pm$0.234 & 1.455\newline$\pm$0.258  & 2.426\newline$\pm$0.207 & 3.080\newline$\pm$0.136\\ \hline
 \multicolumn{2}{c|}{\multirow{2}{*}{SED}}  & 0.003\newline$\pm$0.000  & 0.004\newline$\pm$0.001  & 0.004\newline$\pm$0.000 & 0.004\newline$\pm$0.000  & 0.006\newline$\pm$0.003  & 0.284\newline$\pm$0.308 &  1.283\newline$\pm$0.374 & 2.322\newline$\pm$0.226 & 3.041\newline$\pm$0.188 \\ \hline
 \multirow{4}{*}{\makecell{\textbf{Distr-LB}\\(this paper)}}  &  \multirow{2}{*}{VBF} & 0.004\newline$\pm$0.000  & 0.004\newline$\pm$0.000 & 0.004\newline$\pm$0.000  & 0.004\newline$\pm$0.000  & \textbf{0.005\newline$\pm$0.001}  & \textbf{0.055\newline$\pm$0.070} & 1.039\newline$\pm$0.144 & 1.617\newline$\pm$0.135 & 2.277\newline$\pm$0.096 \\ \cline{2-11}
 & \multirow{2}{*}{VBF+$\log$VBF}  &  0.004\newline$\pm$0.000 & 0.004\newline$\pm$0.000  & 0.004\newline$\pm$0.000 & 0.004\newline$\pm$0.000  & 0.006\newline$\pm$0.004  & 0.116\newline$\pm$0.114  &  \textbf{0.750\newline$\pm$0.063} & \textbf{1.413\newline$\pm$0.083} & \textbf{2.076\newline$\pm$0.096} \\ 
\bottomrule
\end{tabularx}
\label{tab:avg_static}
\end{table}

\subsubsection{MARL Robustness}
\label{app:results-robust}

\begin{table}[tbp]
	\scriptsize
	\centering
	\vskip -.1in
	\caption{Comparison of QoS (mean, $95$th-percentile, and $99$th-percentile task completion time in $s$) when server processing capacity changes over time.}
	\resizebox{\columnwidth}{!}{ 
	\begin{tabular}{c|c|c|c|c|c|c|c}
	\toprule
	\multicolumn{2}{c|}{} & \multicolumn{3}{c|}{Wiki} &  \multicolumn{3}{c}{Static}  \\ \hline
	\multicolumn{2}{c|}{} & \multicolumn{1}{c|}{Mean} &  \multicolumn{1}{c|}{$95th$-percentile} & \multicolumn{1}{c|}{$99th$-percentile} &  \multicolumn{1}{c|}{Mean} &  \multicolumn{1}{c}{$95th$-percentile} & \multicolumn{1}{c}{$99th$-percentile}\\ \hline
	\multicolumn{2}{c|}{WCMP}               & $1.792\pm0.393 $ & $7.534\pm1.817$ & $2.366 \pm 1.685$ &  $1.512\pm0.385 $  & $6.571\pm1.996$ & $1.084 \pm 1.842$\\ 
	\multicolumn{2}{c|}{LSQ}                & $0.453\pm0.178 $ & $1.958\pm0.827$ & $3.482 \pm 1.257$ &  $0.202\pm0.130 $ & $0.975\pm0.617$  & $1.801 \pm 1.064$\\ 
	\multicolumn{2}{c|}{SED}                & $0.340\pm0.268 $ & $1.225\pm0.812$ & $30.600 \pm 6.718$ &  $0.130\pm0.206 $ & $0.519\pm0.571$ & $29.893 \pm 7.042$\\ \hline 
	\multirow{4}{*}{\textbf{QMix-LB}} & MS  & $ 0.373\pm0.177$ & $1.621\pm0.830$ & $4.046 \pm 6.632$ &  $0.144\pm0.112 $ &  $0.663\pm0.523$ & $2.655 \pm 6.899$\\ 
									  & PBF & $ 0.368\pm0.375$ &  $1.529\pm1.581$& $2.436 \pm 1.468$ &  $0.159\pm0.338$ & $0.733\pm1.437$   & $0.974 \pm 1.204$\\ 
									  & VBF & $0.282\pm0.166$  & $1.186\pm0.799$ & $3.187 \pm 1.479$ &  $0.081\pm0.104$ & $0.395\pm0.518$   & $1.654 \pm 1.181$\\ 
							& VBF+$\log$VBF & $0.533\pm0.179 $ & $2.525\pm0.913$ & $4.864 \pm 1.635$ &  $0.266\pm0.129$ & $1.409\pm0.680$   & $3.374 \pm 1.626$\\ \hline
	\multirow{2}{*}{\makecell{\textbf{Distr-LB}\\(this paper)}} & VBF& $0.262\pm0.100$  & $1.086\pm0.454$ & $2.190 \pm 0.792$ & $0.057\pm0.044$ &  $0.305\pm0.234$   & $0.683 \pm 0.510$\\ 
							& VBF+$\log$VBF & $ \mathbf{0.221\pm0.112}$ & $\mathbf{0.895\pm0.530}$ & $\mathbf{1.903 \pm 0.976}$ &$\mathbf{0.039\pm0.057}$ & $\mathbf{0.197\pm0.284}$ & $\mathbf{0.480 \pm 0.650}$  \\
	\bottomrule
	\end{tabular}
	}
	\label{tab:marl_robustness}
	\end{table}

With the rise of elastic and server-less computing, where tenants in data center can share physical resources (\eg CPU, disk, memory), servers can have different processing capacities, which may also change over time dynamically —- because of \eg updated server configuration (upgrading an Amazon EC2 \texttt{a1.xlarge} instance to \texttt{a1.4xlarge}) or resource contention (co-located workloads)~\cite{guo2019limits}.
According to~\cite{silkroad2017}, there are $32\%$ of server clusters in data center that update more than $10$ times per minute based on the measurements collected over $432$ minutes up time in a month.
$3\%$ of clusters have more than $50$ updates perf minute.
Therefore, dynamic changes prevail in real-world data center networks.

Therefore, this section studies the robustness of the proposed distributed RL-based LB framework to react to dynamic changes in server processing speeds, \eg when server VMs are migrated to a new physical architecture. Using the same moderate-scale real-world testbed with $2$ LB agents, additional CPU-bound workloads are applied on the 4-CPU server group starting from $25$s. As depicted in Fig.~\ref{fig:app-robust-timeline}, under heavy Wikipedia traffic, MARL-based LB agents adapt server weights over time and achieves better performance than heuristic LB algorithms -- finishing the same amount of workloads faster, maintaining lower amount of acive number of threads, even when server processing capacity is reduced.
As depicted in Fig.~\ref{fig:app-robust-cdf}, over multiple runs ($10$ runs for each LB algorithm), RL-based LB algorithms effectively achieves lower task completion time in dynamic environments.
They help avoid human intervention and make the LB agents autonomously adapt to the changes in the system.
Table~\ref{tab:marl_robustness} lists the performance of all LB algortihms in terms of the QoS (measured as the average and $95$th-percentile task completion time).

\begin{figure}[tbp]
	\centering
	\begin{subfigure}{0.8\columnwidth}
		\centering
		\includegraphics[width=\columnwidth]{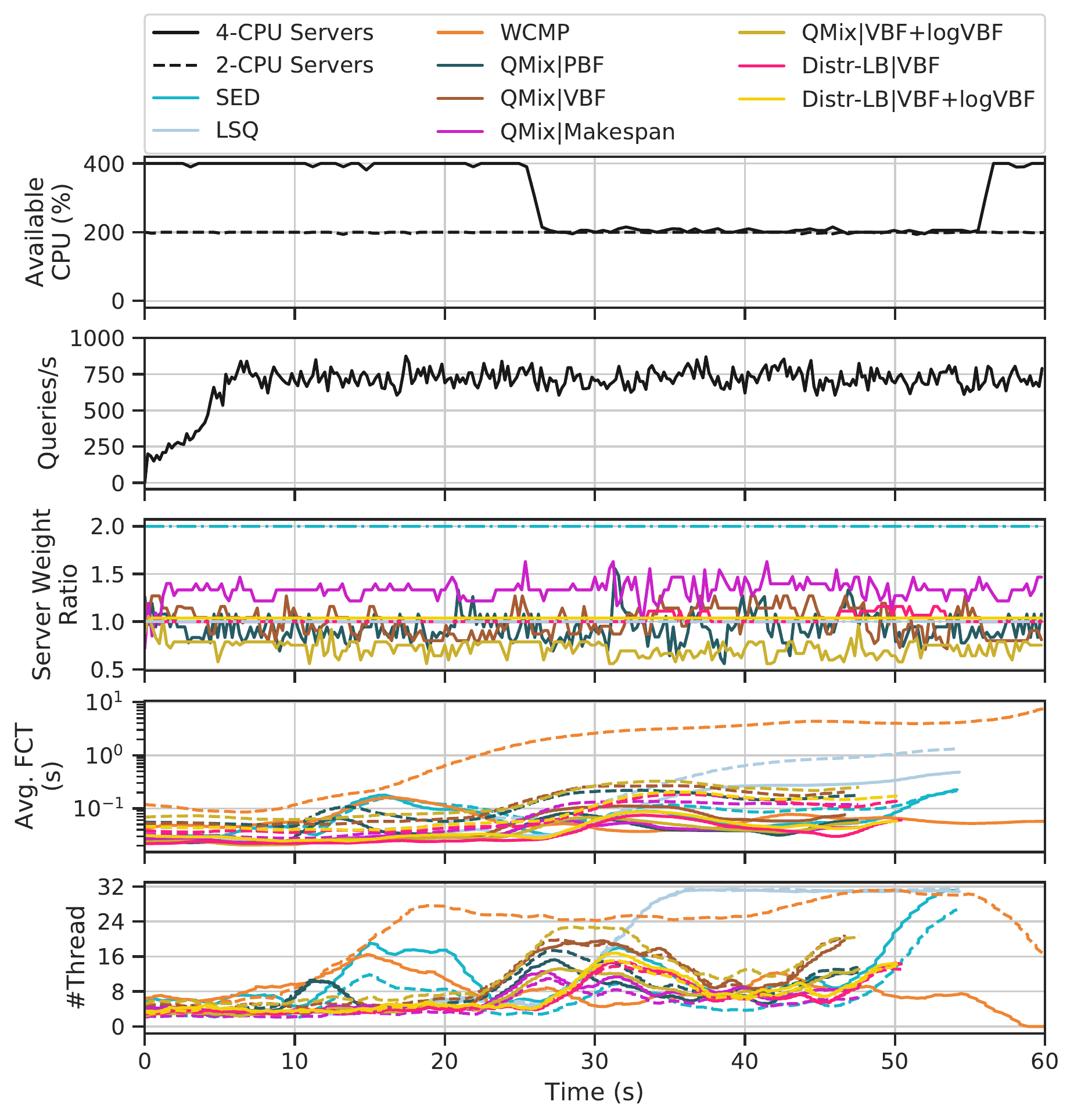}
		\vskip -.1in
		\caption{Additional workloads are applied on servers with 4 CPUs at around $25$s.}
		\label{fig:app-robust-timeline}
	\end{subfigure}
	\hspace{.05in}
	\begin{subfigure}{0.8\columnwidth}
		\centering
		\includegraphics[width=\columnwidth]{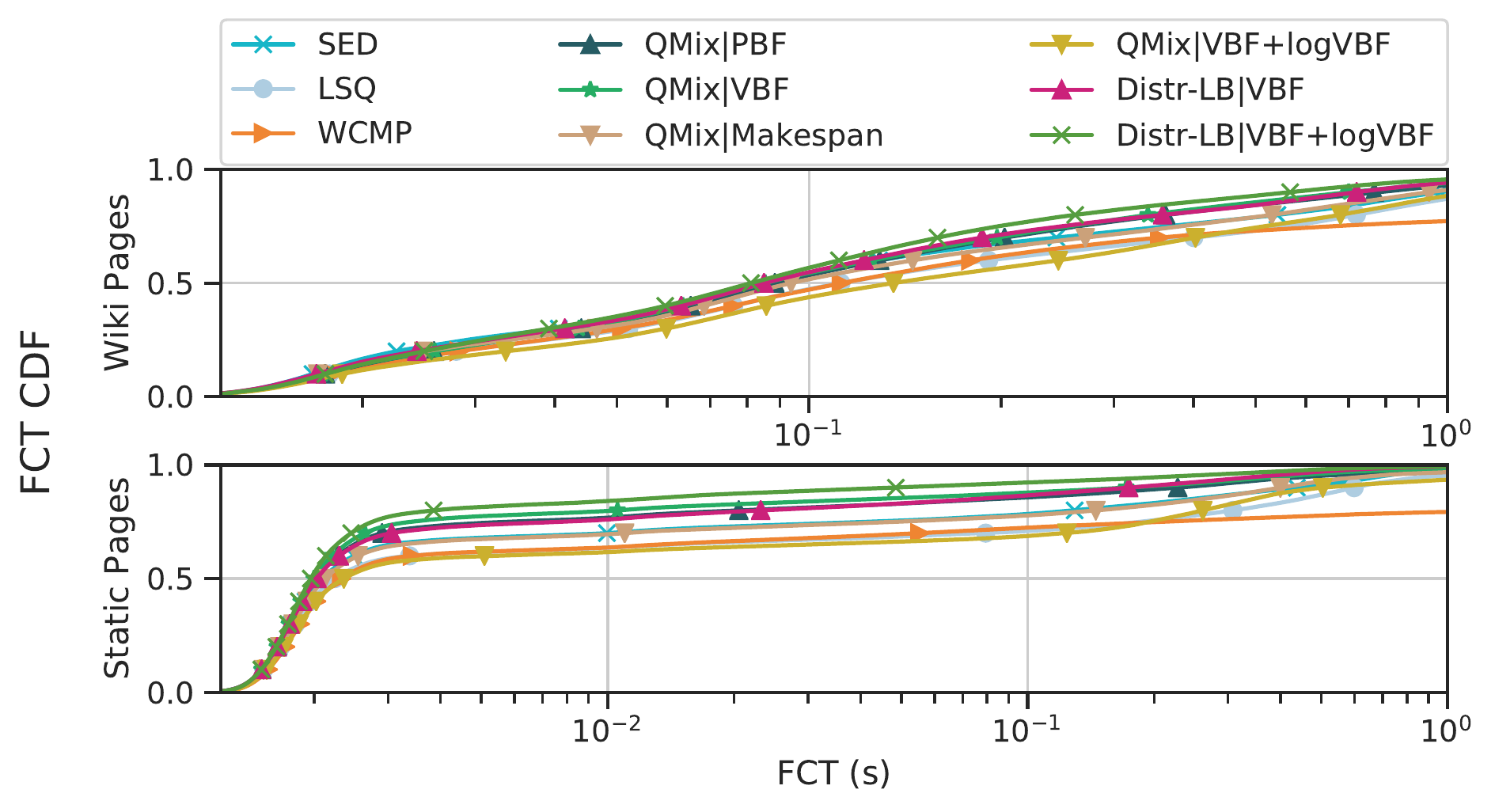}
		\vskip -.1in
		\caption{FCT CDF comparisons for two types of tasks.}
		\label{fig:app-robust-cdf}
	\end{subfigure}
	\caption{Load balancing performance comparison in dynamic environments.}
	\label{fig:app-robust}
	\vskip -.2in
\end{figure}

\end{document}